%% file: tvcg_ivseg_revised.tex
\documentclass[10pt,journal,compsoc]{newIEEEtran}
\usepackage{epsfig}
\usepackage{graphicx}
\usepackage{amsmath}
\usepackage{amssymb}
\usepackage{subfig}
\usepackage{booktabs}
\usepackage{rotating}
\usepackage{multirow}
\usepackage{balance}
\usepackage{enumerate}
\usepackage[ruled,vlined]{algorithm2e}

\newtheorem{theorem}{Theorem}[section]

\newenvironment{proof}[1][Proof]{\begin{trivlist}
		\item[\hskip \labelsep {\bfseries #1}]}{\end{trivlist}}

\newcommand{\qed}{\nobreak \ifvmode \relax \else
	\ifdim\lastskip<1.5em \hskip-\lastskip
	\hskip1.5em plus0em minus0.5em \fi \nobreak
	\vrule height0.75em width0.5em depth0.25em\fi}

\ifCLASSOPTIONcompsoc
  \usepackage[nocompress]{cite}
\else
  \usepackage{cite}
\fi
\ifCLASSINFOpdf
\else
\fi

\begin{document}
%
\title{Segmentation Rectification for Video Cutout \\via One-Class Structured Learning}
%
%
%
%

\author{Junyan~Wang, \IEEEmembership{Member,~IEEE,}
        Sai-Kit~Yeung, Jue~Wang,~\IEEEmembership{Senior Member,~IEEE,} and\\
        Kun~Zhou,~\IEEEmembership{Fellow,~IEEE}
\IEEEcompsocitemizethanks{\IEEEcompsocthanksitem Junyan Wang is with Doheny Eye Institute at University of California, Los Angeles, CA 90033, USA\protect\\
E-mail: {ejywang@ucla.edu}
\IEEEcompsocthanksitem Sai-Kit Yeung is with the
Pillar of Information Systems Technology and Design, Singapore University of Technology and Design, Singapore, 487372.\protect\\
E-mail: e-mail:saikit@sutd.edu.sg
\IEEEcompsocthanksitem Jue Wang is with Adobe Research, Seattle, WA 98103, USA.\protect\\
E-mail: e-mail:juewang@ieee.org
\IEEEcompsocthanksitem Kun Zhun is with the State Key Laboratory of CAD\&CG, Zhejiang
University, Hangzhou, China, 310058.\protect\\
E-mail: kunzhou@acm.org.
}
\thanks{Manuscript received ; revised }}

%
%

\markboth{}%
{Shell \MakeLowercase{\textit{et al.}}: Bare Demo of IEEEtran.cls for Computer Society Journals}
%



\IEEEtitleabstractindextext{%
\begin{abstract}
Recent works on interactive video object cutout mainly focus on designing dynamic foreground-background (FB) classifiers for segmentation propagation. However, the research on optimally removing errors from the FB classification is sparse, and the errors often accumulate rapidly, causing significant errors in the propagated frames. In this work, we take the initial steps to addressing this problem, and we call this new task \emph{segmentation rectification}. Our key observation is that the possibly asymmetrically distributed false positive and false negative errors were handled equally in the conventional methods. We, alternatively, propose to optimally remove these two types of errors. To this effect, we propose a novel bilayer Markov Random Field (MRF) model for this new task. We also adopt the well-established structured learning framework to learn the optimal model from data. Additionally, we propose a novel one-class structured SVM (OSSVM) which greatly speeds up the structured learning process. Our method naturally extends to RGB-D videos as well. Comprehensive experiments on both RGB and RGB-D data demonstrate that our simple and effective method significantly outperforms the segmentation propagation methods adopted in the state-of-the-art video cutout systems, and the results also suggest the potential usefulness of our method in image cutout system.\\

{\centering
		\includegraphics[height=3cm]{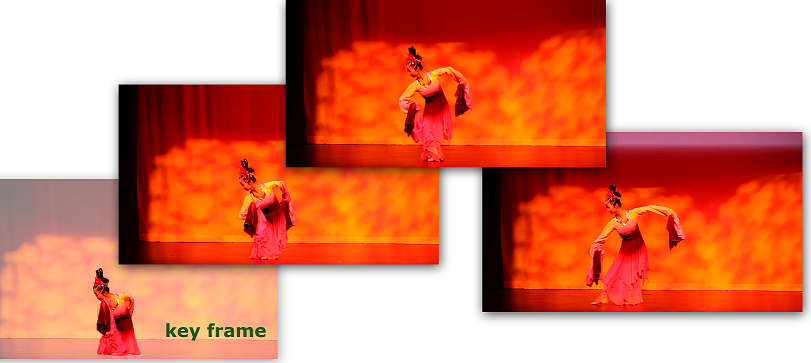}
		\includegraphics[height=3cm]{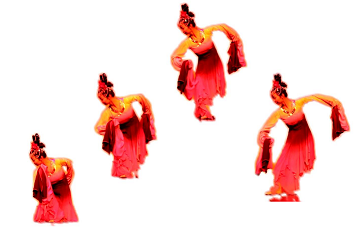}
		\includegraphics[height=3cm]{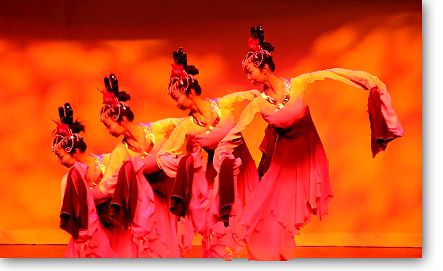}\setcounter{figure}{-1}    
		\captionof{figure}{Given a keyframe segmentation provided by the user (left), our approach generates accurate object cutout results in subsequent frames fully automatically (middle), which can be used for creating a novel compositing (right).}
		\par
}
\end{abstract}

}

\maketitle
\IEEEpeerreviewmaketitle

\input{introduction_TVCG_revised}



\input{mrfmodel_TVCG_revised}
\input{learning_TVCG_revised}
\input{results_CVPR_revised}

\input{discussion_CVPR_revised}

%
%

{
	\bibliographystyle{ieeetr}
	\bibliography{VideoSeg,Matching,MRFSeg,LevelSetActiveContours,OtherSeg,ml}
}

\end{document}

%% file: introduction_TVCG_revised.tex
\IEEEraisesectionheading{\section{Introduction}\label{sec:introduction}}

\IEEEPARstart{V}{ideo} cutout, as one of the most successful applications of computer vision for video editing and compositing, has gained much attention from the computer graphics community~\cite{chuang2002video,agarwala2004keyframe,Li05VideoGCut_SIGGRAPH,Wang05IVC_SIGGRAPH,Bai09VideoSnapCut_SIGGRAPH,Zhong2012UDC_SIGGRAPHAsia}. While practically useful systems have been developed, some fundamental problems still remain unattended. In this paper, we address video cutout from a newly identified fundamental aspect. 


\subsection{Related works}
The latest video object cutout systems~\cite{Bai09VideoSnapCut_SIGGRAPH,bai2010dynamic,zhong2010transductive,Zhong2012UDC_SIGGRAPHAsia,Fan2015JumpCut} comprise three major steps: (1)  \textbf{Keyframe segmentattion.} performing keyframe image segmentation and refinement; (2) \textbf{Foreground-background classification.} performing classification on other frames given the keyframe segmentation; (3)  \textbf{Segmentation refinement.} converting the classifier outputs to final cutout results on non-keyframes. Steps 2 and 3 are usually iteratively applied to subsequent frames until the user creates a new keyframe due to occurrence of visible errors.  

The keyframe segmentation step can be done using interactive single image segmentation techniques. In foreground-background classification, with the help of motion estimation, foreground classifiers at different scales are constructed/trained from the segmented object in the current frame, and then applied to other frames. The classifiers generate a soft foreground probability map, which is incorporated into a segmentation model in the segmentation refinement step to generate the final object mask on the non-keyframes. 

\begin{figure}[t]
	\centering
	{\small\begin{tabular}{
				@{\hspace{0mm}}c@{\hspace{0mm}}c@{\hspace{0mm}}c
			}
			\includegraphics[width=0.32\columnwidth]{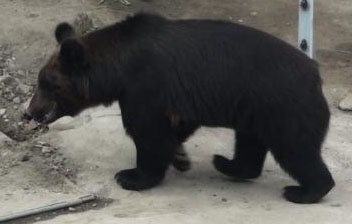}&\hspace{-1.5mm}
			\includegraphics[width=0.32\columnwidth]{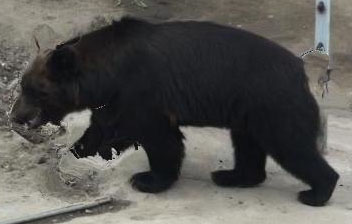}&
			\includegraphics[width=0.32\columnwidth]{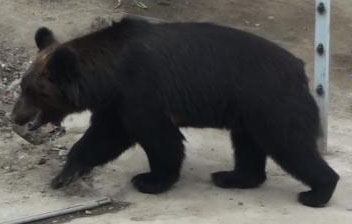}\\
			frame $t$ & warped frame $t$ & frame $t+1$\\
			\includegraphics[width=0.32\columnwidth]{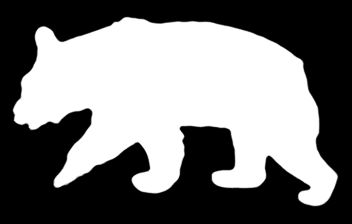}&\hspace{-1.5mm}
			\includegraphics[width=0.32\columnwidth]{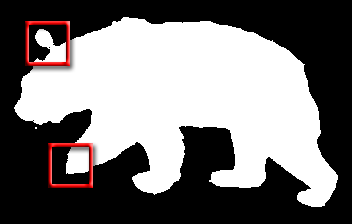}&
			\includegraphics[width=0.32\columnwidth]{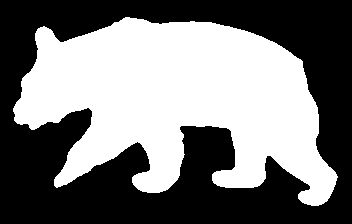}\\
			ground truth of $t+1$& Result by \cite{Zhong2012UDC_SIGGRAPHAsia} & Our method  \\
		\end{tabular}}
		\caption{The need of optimal segmentation rectification.}\label{FIG:ProbExample}\vspace{-0.5cm}
	\end{figure}
\begin{figure*}[t]
	\centering
	\includegraphics[width=0.95\linewidth]{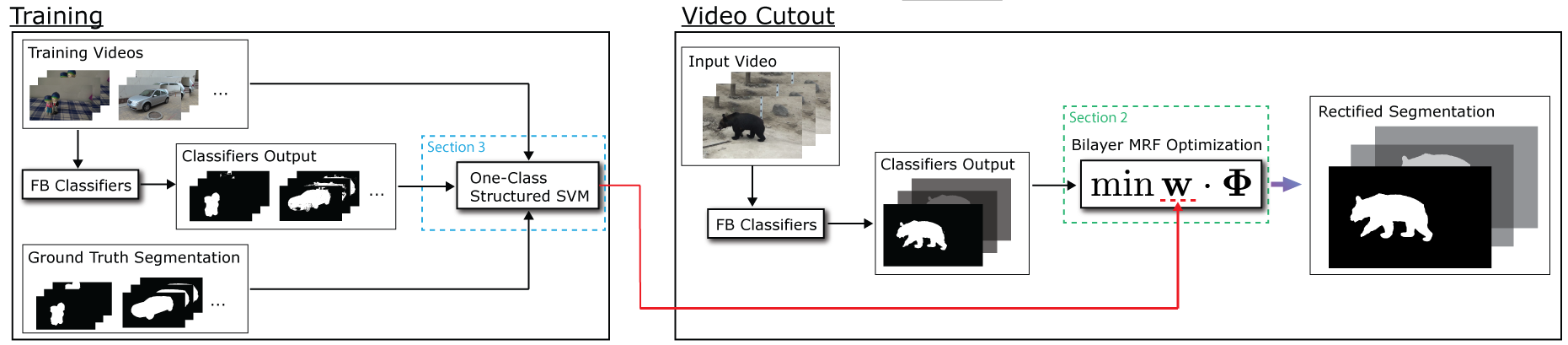}
	\caption{An overview of our approach.}\label{FIG:flowchart}
\end{figure*}
The Foreground-background classification and segmentation refinement are often considered as a single module called \emph{segmentation propagation}, and most previous works focus only on designing new foreground-background classifiers, while little attention has been paid to the optimal estimation of the segmentation given the classifier output. In Video Snapcut~\cite{Bai09VideoSnapCut_SIGGRAPH}, Bai et al. applied the conventional Markov random field (MRF) to the classification output to refine the result. More recently, Zhong et. al~\cite{Zhong2012UDC_SIGGRAPHAsia} applied matting directly after classification. The matting step behaves like random walk segmentation~\cite{grady2006random} and the latter is closely related to the MRF model~\cite{Sinop2007seeded}. We observe that when the errors from the classifiers tend to bias toward either the false positive or false negative error, the common MRF or matting, which treat the two types of errors equally, would fail to remove the errors satisfactorily. As an example, Figure~\ref{FIG:ProbExample} shows that common method could over shrink the spurious foreground regions produced by the classifier. 

More recently, Fan et al.\cite{Fan2015JumpCut} proposed a novel method for propagating segmentation to non-successive frames in order to handle large object displacement. The propagated segmentation mask was refined using geodesic active contour (GAC) model \cite{Caselles97GAC} and the level set method \cite{OsherSethian88Fronts}, rather than MRF with graph cuts. Similar to the MRF based frameworks, GAC with level set method has also not been used to handle the possibly asymmetrically distributed FP and FN. Besides, other video cutout frameworks have been proposed \cite{Tong2011video,FuHongbo2012EXCOL,zhang2015efficient}. Our contribution is parallel to these directions towards accurate and user-friendly video cutout. 

\subsection{Contributions}
In this paper, we address the possibly asymmetrically distributed FP and FN errors in the output of foreground-background classifier, which we call {\em segmentation rectification}. The significance of this subproblem in the context of video cutout is first identified to the best of our knowledge. 

Our contribution is twofold. First, we propose a novel bilayer MRF in which the data term can treat the false positive and false negative errors from any given classifier differently using separate weights. Second, we propose a novel one-class structured support vector machine (OSSVM) model to learn the weights, as a computationally more favorable alternative to the conventional two-class structured SVM (2CSSVM) frameworks~\cite{Taskar05LargeMargin,Tsochantaridis05LargeMarginSSVM}. We further establish the conditional equivalence between the OSSVM and the conventional (2CSSVM). This theoretical justification of OSSVM is also new in the context of structured learning. Figure~\ref{FIG:flowchart} illustrates the flowchart of our method. 



Our proposed method for segmentation rectification adapts to different classifiers and achieves significant improvement on error reduction over previous methods~\cite{Zhong2012UDC_SIGGRAPHAsia,Bai09VideoSnapCut_SIGGRAPH} in segmentation propagation in the experiments. Note that the {confidence map} adopted in~\cite{Zhong2012UDC_SIGGRAPHAsia} can be used to eliminate unreliable/ambiguous results from classifier output. However, it does not tell where the classifier is overconfident. Our method can remove the error in the classification regardless of the confidence of the classification.
\subsection{Organization}
The rest of the paper is organized as follows. Section~\ref{sec:bilayerMRF} derives our bilayer MRF model, Section~\ref{SEC:OSSVM} describes the training process and Section~\ref{sec:practical} discuses the practical concerns for video cutout. Section~\ref{sec:experiment} details our experiment. Finally, we conclude our work and discuss about the difference between video cutout and other video segmentation tasks in Section in Section~\ref{sec:conclusion}.

%% file: mrfmodel_TVCG_revised.tex
\section{Modeling segmentation rectification}
\label{sec:bilayerMRF}

\subsection{Segmentation refinement in video cutout}


In conventional video cutout systems~\cite{Bai09VideoSnapCut_SIGGRAPH,Zhong2012UDC_SIGGRAPHAsia}, the classifier output is refined by using the MRF-based segmentation model. The MRF model can be written as:
\begin{equation}\label{eqn:MRF_prior}
\min_{f} \sum_{p\in\mathcal{P}} U_{p}(f_p) + \sum_{\{p,q\}\in \mathcal{N}} V_{pq}(f_p,f_q), 
\end{equation}
where $p$ refers to a  pixel, $\mathcal{P}$ is the set of all pixels, $f_p$ is the pixel label and $\mathcal{N}$ is a neighbourhood system. $U_p$ and $V_{pq}$ are the conventional unary and pairwise terms. The unary term can be used to represent the hard constraint given by the user, such as the seeds of foreground and background regions, or it can be a region model or a shape prior. The pairwise potential is often used to model the object boundaries, and it has also been used to represent advanced priors in segmentation~\cite{Veksler08StarShapePrior}.

The unary term for incorporating foreground-background model can be written in the following form:
\begin{equation}\label{EQ:D_TPFP}
\begin{split}
U_p(f_p) &= \sum_{p\in\mathcal{P}} (f_p-h_p)^2 \\
&= \sum_{p\in\mathcal{P}} f_p+h_p-2f_ph_p\\
&=\sum_{p\in\mathcal{P}} h_p(1-f_p)+f_p(1-h_p)\\
&=\sum_{p\in\mathcal{P}} \overline{f_p}h_p+f_p\overline{h_p}.
\end{split}
\end{equation}
where $h_p$ is the classifier output (or probability map that gives the classifier output), $h_p$ and $f_p$ are both binary, $ \overline{f_p} = 1-f_p$ and $\overline{h_p} = 1-h_p$. The unary term is a pixelwise shape distance between the label $f$ and the classifier output $h$. The above equation also provides a decomposition of the unary term, which implies that the unary term above can be naturally related to the two types of errors in segmentation, i.e., the false positives (FP) (background that wrongly considered as foreground) $\overline{f_p}h_p$, and the false negatives (FN) (missing foreground) $f_p\overline{h_p}$, as illustrated in Figure~\ref{FIG:SSM}. FPR is defined as 
\begin{equation}
FPR = \frac{\#\hbox{ of wrongly classified foreground pixels}}{\hbox{total \# of pixels}}
\end{equation}
and FNR is defined as 
\begin{equation}
FNR = \frac{\#\hbox{ of wrongly classified background pixels}}{\hbox{total \# of pixels}}.
\end{equation}
\begin{figure}[t]
	\centering
	\subfloat{\includegraphics[height=3.0cm]{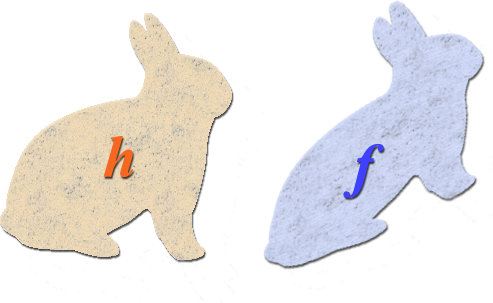}}
	\subfloat{\includegraphics[height=3.0cm]{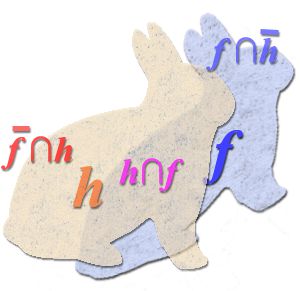}}
	\caption{Illustration of the two types of errors in segmentation propagation.}\label{FIG:SSM}
\end{figure}

\subsection{A case study on the classification error}\label{sec:errana}
Here, we conduct a quantitative study on the classification errors produced by the state-of-the-art foreground-background classifier for video cutout~\cite{Zhong2012UDC_SIGGRAPHAsia}. In this study, we perform segmentation propagation using Zhong et al.'s classifier~\cite{Zhong2012UDC_SIGGRAPHAsia} on all consecutive frame pairs in their training dataset, and we compute the false positive ratio (FPR) and false negative ratio (FNR) for each target frame. The quantitative results are shown in Figure~\ref{FIG:StatEvid}: the FPR is about {\em $11.22$ times greater} than the FNR, implying that the two terms in Eq. (\ref{EQ:D_TPFP}) should not be considered as equally important in the model. This case study disproves the universal fidelity of the unary term based on symmetric distance measure in the conventional MRF model. There may be multiple complicated causes of this phenomenon and we omit the in-depth analysis on this specific classifier, since our method is generically applicable to removing errors from any classifier. 
\begin{figure}
	\centering
	\includegraphics[width=0.7\columnwidth]{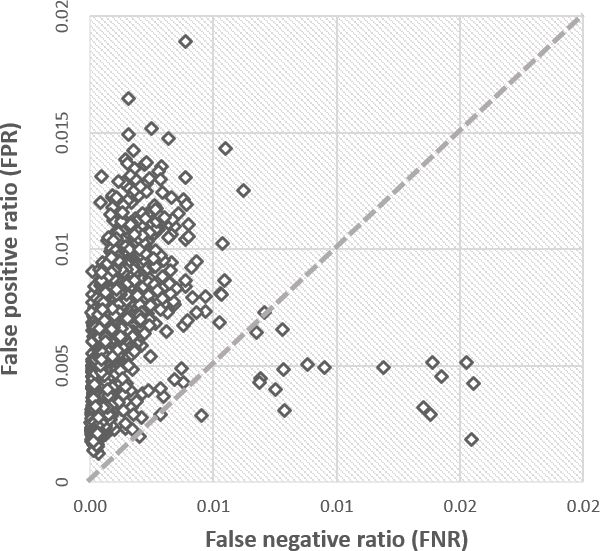}
	\caption{FPR vs. FNR from Zhong et al.'s FB classifier \cite{Zhong2012UDC_SIGGRAPHAsia}}\label{FIG:StatEvid}
\end{figure}

\subsection{A generic shape distance function}
Our segmentation rectification method is inspired by the shape-prior based MRF segmentation model~\cite{Freedman2005ShapePriorGC,Vu2008Shape} in which shape distance is used in the segmentation model for handling occlusion and background clutter. In this work, we view the data term as shape distance and we show why and how we could reformulate this shape distance for segmentation rectification.

From Eq. (\ref{EQ:D_TPFP}), we observe that the common shape distance used in the data term of the MRF uniquely breaks down into FP and FN with equal weights. To handle the possibly asymmetrically distributed FP and FN errors, we propose the following generic shape distance function to model the relationship between the classifier output and the true segmentation:
\begin{equation}\label{EQ:SDist}
S_w(f,h) = \sum_{\{p,q\}\in\mathcal{N}_{fh}} w^{outside}_{pq}\overline{f_p}h_{q}+w^{inside}_{pq}f_p\overline{h_{q}},
\end{equation}
where $\mathcal{N}_{fh}$ denotes the neighborhood system across $f$ and $h$. In this formulation, $w^{outside}_{pq}$ and $w^{inside}_{pq}$ are two unknown data-dependent balancing weights. 

The neighborhood system across $f$ and $h$ yields a novel graph as visualized in Figure~\ref{FIG:BilayerMRF}. The graph is in a bilayer structure: one layer is defined on the image to represent the unknown segmentation label $f^t$ at frame $t$, and the other layer is used to represent the propagated label $h^t$. 
\begin{figure}
  \centering
  \includegraphics[width=0.85\linewidth]{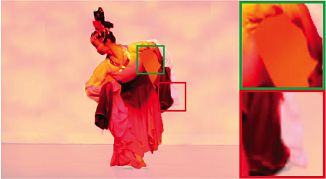}
  \caption{FP (green box) and FN (red box) from rotobrush on a frame in the ``Chinese dance'' sequence.}\label{FIG:dance_FP_FN}
\end{figure}
We have the following observation for different configurations of weights when $p=q$:
\vspace{-5pt}
\begin{center}
\def\arraystretch{1.2}

\small{\begin{tabular}{c|c}
\toprule
\multicolumn{1}{c|}{\textbf{weights}}& \textbf{effect}\\
\hline
 $w^{outside}_{pq}$ $=$ $w^{inside}_{pq}$ & Penalize FP and FN equally\\
$w^{outside}_{pq}$ $>$ $w^{inside}_{pq}$ & Penalize more FP than FN~~~~\\
 $w^{outside}_{pq}$ $<$ $w^{inside}_{pq}$ & Penalize more FN than FP~~~~\\
\bottomrule
\end{tabular}}
\end{center}
This means the FP and FN can be treated differently with proper weights. Besides, by considering all the $q$s in $\mathcal{N}_{fh}$ in the formulation, the noise in the classification result can also be reduced, and the resultant $S_{w}$ can be viewed as the local average distance.
\begin{figure}
  \centering
  {\includegraphics[width=0.7\columnwidth]{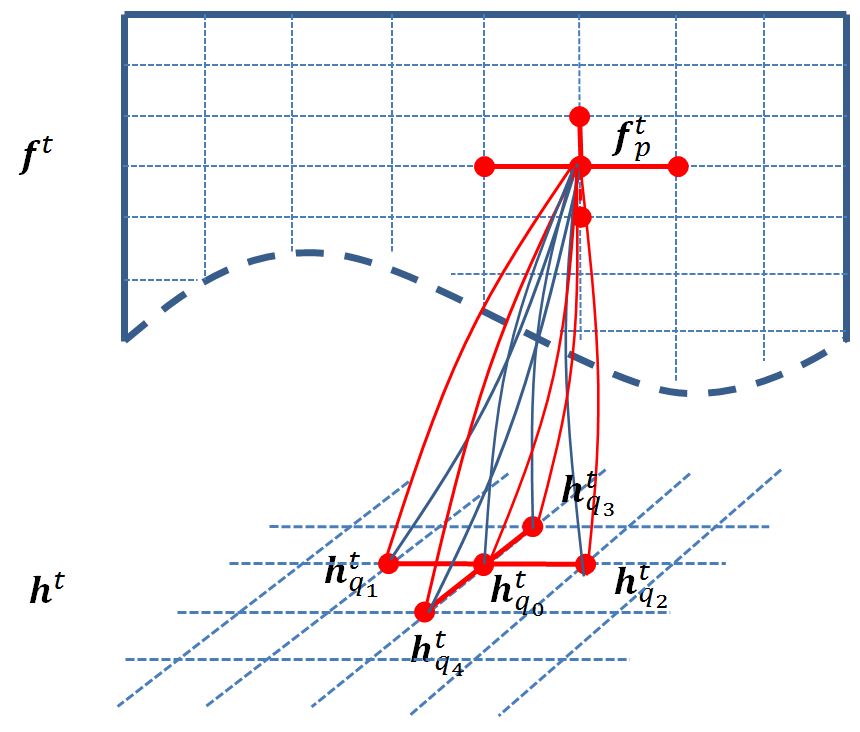}}
  \caption{Graph structure of the proposed bilayer MRF model. $p$ in $f^t$ will connect to five neighboring $q$s in $h^t$}\label{FIG:BilayerMRF}
\end{figure}

\subsection{The MRF model for segmentation rectification}


We can now plug our generic shape decision function in Eq.~(\ref{EQ:SDist}) into the conventional MRF for segmentation propagation to arrive at a novel model for segmentation rectification:


\begin{equation}\label{EQ:Bi-MRF}
\min_{f^{t}} S_w(f^t,h^t)+\sum_{\{p,p'\}\in\mathcal{N}}\delta_e(f^{t}_p,f^t_{p'}),
\end{equation}
where $h^t,f^{t}$ are the hypothesis segmentation label and the unknown label at frame $t$, $\mathcal{N}$ is the neighborhood system for $f^t$. 

To comply with the standard graph-cuts representation of MRF energy, we may rewrite $S_w(f^t,h^t)$ as follows
\begin{equation}
S_w(f^t,h^t) = \sum_{\{p,q\}\in\mathcal{N}_{fh}} \delta_s({f^{t}_p},h^t_q)
\end{equation}
where $\mathcal{N}_{fh}$ denotes the neighborhood system for the graph defined on $h^t$ and $f^{t}$, $\delta_s(f^t_p,h^t_q)$ is defined according to the generic shape distance measure, shown in Eq. (\ref{EQ:SDist}), as:
\begin{equation}
\delta_s(f^t_p,h^t_q) = \left\{\begin{array}{lr}
                  w^{inside}_{pq},  & \hbox{if }h^t_q=0,f^t_p=1\\
                  w^{outside}_{pq}, & \hbox{if }h^t_q=1,f^t_p=0
                \end{array}\right.
\end{equation}
where $w^{inside}_{pq}$ and $w^{outside}_{pq}$ are to be determined. $\delta_e(f^{t}_p,{f^{t}_{p'}})$ in Eq.~(\ref{EQ:Bi-MRF}) corresponds to the boundary edge model proposed in~\cite{Mishra2012AVS}, which has been shown to be effective for interactive segmentation:
\begin{equation}\label{EQ:edgeterm}
\delta_e(f^{t}_p,f^t_{p'}) =  w^{edge}\cdot w^{e}_{pq} \big|f^{t}_p-f^t_{p'}\big|,
\end{equation}
where $w^{edge}$ is the model weight to be determined. $w^{e}_{pp'}$ in \cite{Mishra2012AVS} was defined as:
\begin{equation}\label{EQ:weq}
w^{e}_{pp'}=\left\{\begin{array}{lr}
                    \exp(-5I_e(p,p')),~ & I_e(p,p')\neq0 \\
                    20,~ & \hbox{Otherwise}
                  \end{array}
\right.
\end{equation}
where $I_e(p,p')$ is 1 if either $p$ or $p'$ is on edge. Otherwise, value is 0. We also normalize the weight by $w^{e}_{pp'}=w^{e}_{pp'}/\max_{\{p,p'\}\in\mathcal{N}}(w^{e}_{pp'})$.

The input to the graph cuts solver includes the graph structure, the estimated segmentation of the current image ($I^t$) and the edge map ($I_e$) of the current image. The output is the rectified segmentation ($f^t$). 

The energy function in our bilayer MRF model in Eq.~(\ref{EQ:Bi-MRF}) can be rewritten in the following compact form:
\begin{equation}\label{EQ:Bi-MRF_linear}
E(f^t|\mathbf{w},h^t,w^e_{pp'}) 
= \mathbf{w}\cdot \mathbf{\Psi}(h^t,w^e_{pp'},f^t),
\end{equation}
where $\cdot$ is the vector dot product, $w^e_{pp'}$ is defined in Eq.~(\ref{EQ:weq}), the weight vector is defined as $
\mathbf{w}=\big[w^{edge}, w^{inside}_{pq}, w^{outside}_{pq}|\{p,q\}\in \mathcal{N}_{fh}\big]
$, and 
\begin{equation}\label{EQ:DefinePsi}
\begin{split}
\mathbf{\Psi} = \left(\begin{array}{lr}
                                              \sum_{pp'}w^{e}_{pp'} \big|f^{t}_p-{f^{t}_{p'}}\big|,&~\{p,p'\}\in\mathcal{N}~~ \\
                                              \sum_{pq}(1-h^t_q)f^{t}_p,&~\{p,q\}\in\mathcal{N}_{fh} \\
                                              \sum_{pq}h^t_q(1-f^{t}_p),&~\{p,q\}\in\mathcal{N}_{fh}
                                            \end{array}\right),
\end{split}
\end{equation}
Note that there can be multiple terms for $\{p,q\}\in \mathcal{N}_{fh}$ as shown in Figure \ref{FIG:BilayerMRF}. Throughout this paper we mainly consider the following parameterization of $\mathbf{w}$: $\mathbf{w} = [w_1,w_2,...,w_{11}]^T$$ = [w^{edge},~ w^{inside}_{p,q_0},~w^{inside}_{p,q_1}, ~w^{inside}_{p,q_2},~w^{inside}_{p,q_3},  ~w^{inside}_{p,q_4}, ~w^{outside}_{p,q_0}, $ $ w^{outside}_{p,q_1},~ w^{outside}_{p,q_2},~w^{outside}_{p,q_3},~w^{outside}_{p,q_4}]^T$.


%% file: learning_TVCG_revised.tex
\begin{figure}[t]
  \centering
  \includegraphics[width=0.95\columnwidth]{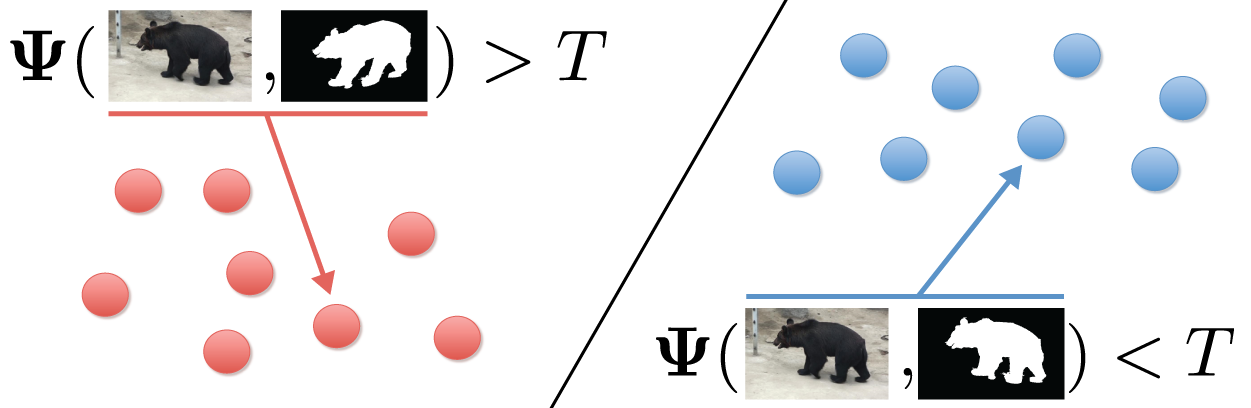}
  \caption{Idea of SSVM. Red dots represents the ``bad'' class, and the blue dots represent the ``good'' class. The black line separating the two classes is the underlying classifier $\Psi=T$, where $T$ is a thresholding constant. The value of $T$ is implicit in our method.}\label{fig:SSVM}
\end{figure}

\section{Learning the optimal model for segmentation rectification with one-class structured SVM}\label{SEC:OSSVM}

Based on the previous observations, it becomes crucial to determine the optimal parameters in the proposed MRF model. One popular method for this task is known as 
\emph{structured learning}~\cite{Taskar05LargeMargin,Tsochantaridis05LargeMarginSSVM,Szummer08LearnCRFbyGC,Joachims2009CuttingPlaneSSVM}. The basic idea is to consider the MRF parameter learning problem as a classification problem where good segmentations form one class and bad segmentations are the other class. This idea is illustrated in Figure \ref{fig:SSVM}. However, this framework requires searching for negative label samples $\mathbf{f}_k$, or the worst case~\cite{Taskar05LargeMargin,Tsochantaridis05LargeMarginSSVM,Szummer08LearnCRFbyGC,Joachims2009CuttingPlaneSSVM}, which can be time-consuming. To remove the need for negative samples, we propose to apply the one-class SVM, instead of the conventional two-class formulation~\cite{Taskar05LargeMargin}, to the structured learning problem. The one-class SVM only requires samples from one class, e.g. positive class, for training~\cite{Scholkopf01oneclassSVM,chen2001one,manevitz2002one}. Thus, the resultant one-class structured SVM (OSSVM) will also only require the positive samples which are the images paired with ground truth segmentations.

\subsection{The two-class structured SVM}
Before we present our model, we briefly describe the conventional two-class structured SSVM (2CSSVM). We begin with the generic compact form of MRF. The MRF energy for one image can be written as an inner product form w.r.t. the weights $\mathbf{w}$:
\begin{equation}
E(f|\mathbf{w},\mathbf{x}) 
                = \mathbf{w}\cdot \mathbf{\Psi}(\mathbf{x},\mathbf{f})
\end{equation}
where $\mathbf{w}$ is defined in Eq. (\ref{EQ:Bi-MRF_linear}), and $\mathbf{x}=\{h^t,w_{pp'}^e\}$, $\mathbf{f}=f^t$ to be consistent with Eq. (\ref{EQ:Bi-MRF_linear}). Note that this general notations adopted in this subsection allows us to easily extend our approach to other MRF models.

We expect the MRF energy to be as small as possible for the ground truth $\mathbf{f}^*$ and we expect it to be as large as possible for other non-ideal $\mathbf{f}$. 
This principle can be used for learning the weight vector $\mathbf{w}$ from data~\cite{Taskar05LargeMargin,Tsochantaridis05LargeMarginSSVM}, and it can be expressed as:\vspace{-3mm}
\begin{equation}\label{EQ:2C_SSVM}
\begin{split}
&\min_{\mathbf{w},b,\vec{\xi}}~ {1\over2}\|\mathbf{w}\|^2 + {C\over N}\sum_{k=1}^N \xi_k \\
& \begin{array}{rl}
                  \hbox{s.t.:}&~ 2\mathbf{w}\cdot\mathbf{\Psi}(\mathbf{x}_k,\mathbf{f}_k)-b\geq +1 - \xi_k 
                   \\
                   &~2\mathbf{w}\cdot\mathbf{\Psi}(\mathbf{x}_k,\mathbf{f}_k^*)-b\leq -1 + \xi_k\\ 
                   &~k=1,2,3,...,N,
\end{array}
\end{split}
\end{equation}
where $C$ is a constant, $k$ is the sample id, $N$ is the number of samples, the constants $b$ is a bias in conventional decision function and it's unnecessary in segmentation, and $\xi_k$ is a lack variable that tolerates errors in the training data. In the above model, the margin that separates the positive and negative samples are maximized by minimizing $\|\mathbf{w}\|^2$.

Since the positive and negative samples are paired, the respective constraints for each pair of samples can be combined to yield the following constraint: 
\begin{equation}\label{EQ:SSVM_Constr}
                   \mathbf{w}\cdot\mathbf{\Psi}(\mathbf{x}_k,\mathbf{f}_k)-\mathbf{w}\cdot\mathbf{\Psi}(\mathbf{x}_k,\mathbf{f}^*_k)\geq \Delta(\mathbf{f}_k,\mathbf{f}^*_k) - \xi_k
\end{equation}
Directly combining the constraints in Eq. (\ref{EQ:2C_SSVM}) gives $\Delta = 1$. For segmentation problem, $\Delta$ is a specialized cost and is often set as $\Delta = \mathtt{mean}(|\mathbf{f}_k-\mathbf{f}_k^*|^2)$~\cite{Nowozin2011SL_review}. With the additional requirements of positiveness and boundedness of $\mathbf{w}$, the above yields the conventional SSVM of the following form~\cite{Tsochantaridis05LargeMarginSSVM,Taskar05LargeMargin,Szummer08LearnCRFbyGC,Joachims2009CuttingPlaneSSVM}. 
\begin{equation}\label{EQ:2C_SSVM1}
\begin{split}
\min_{\mathbf{w},\vec{\xi}}~& {1\over2}\|\mathbf{w}\|^2 + {C\over N}\sum_{k=1}^N {\xi}_k \\
\hbox{s.t.:}~&  \forall k,~  \mathbf{w}\cdot(\mathbf{\Psi}(\mathbf{x}_k,\mathbf{f}_k)-\mathbf{\Psi}(\mathbf{x}_k,\mathbf{f}^*_k))\geq \Delta_k - \xi_k, \\
&  \sum_i w_i=1, \mathbf{w}\geq 0. 
\end{split}
\end{equation}
where $\Delta_k =\Delta(\mathbf{f}_k,\mathbf{f}^*_k) $. Note that we further imposed a normalization constraint for the weights.

\subsection{One-class structured SVM}

By simply dropping the constraint for non-ideal segmentation $\mathbf{f}_k$ in Eq. (\ref{EQ:2C_SSVM}) and removing the irrelevant parameter $b$, we obtain
\begin{equation}\label{EQ:OSSVM}
\begin{split}
\min_{\mathbf{w},\vec{\varepsilon}}~& {1\over2}\|\mathbf{w}\|^2 + {C\over N}\sum_{k=1}^N \varepsilon_k  \\
\hbox{s.t.:}~&  \forall k,~ \mathbf{w}\cdot\mathbf{\Psi}(\mathbf{x}_k,f^*_k)\leq -1+\varepsilon_k, \\
&  \sum_i w_i=1, \mathbf{w}\geq 0. 
\end{split},
\end{equation}
where we used $\varepsilon_k$ instead of $\xi_k$ to differentiate from the original SSVM formulations. We call this model one-class structured support vector machine (OSSVM), since it does not reply on $f_k$. This optimization problem is well-known as one class support vector machine, and it has been thoroughly studied previously in the classification literature ~\cite{Scholkopf01oneclassSVM}. Here we further establish the rationale of the OSSVM in the context of structured learning. 

We observe that the OSSVM, requiring only ground-truth masks, is conditionally equivalent to the conventional two-class SSVM model in which both of non-ideal segmentations and ground-truth segmentations are used. The formal statement is as follows.
\begin{theorem}\label{THM:2COEQ}
The OSSVM model in Eq. (\ref{EQ:OSSVM}) is identical to the two class SSVM model in Eq. (\ref{EQ:2C_SSVM1}) if both of the following conditions are true:
\begin{enumerate}[(I)]
	\item $\Delta(\mathbf{f}_k,\mathbf{f}^*_k)=1$;
	\item $\forall k, \mathbf{\Psi}(\mathbf{x}_k,f_k)=\mathbf{b}_k$, where $\mathbf{b}_k$ is an arbitrary constant vector with equal elements. Its total sum is denoted by ${B}_k$.
\end{enumerate}
\end{theorem}
\begin{figure}[!t]
	\centering
	\includegraphics[width=0.95\columnwidth]{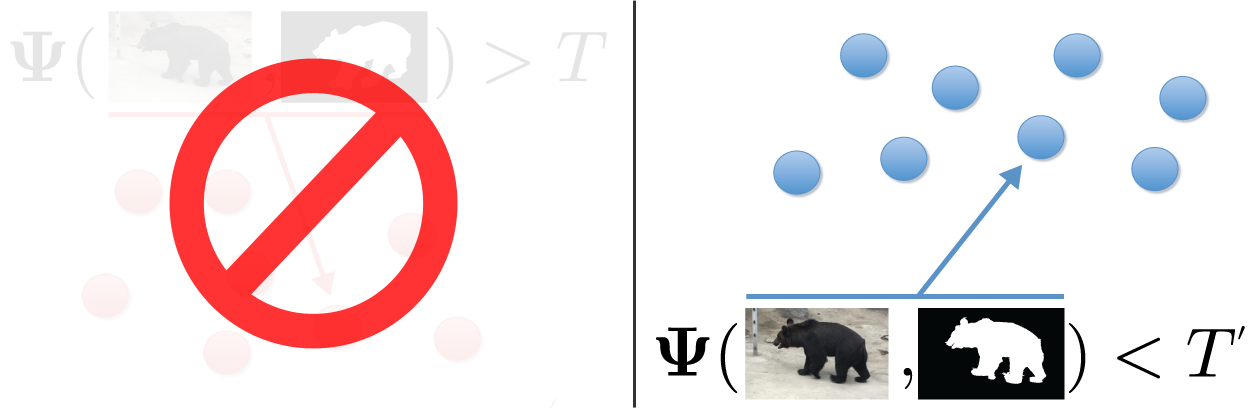} 
	\caption{Idea of OSSVM. The threshold line, $\Psi=T'$, is estimated based on the ``good'' class only. The value of $T'$ is implicit in our method.}\label{fig:OSSVM}
\end{figure}

There are obviously infinitely many such non-ideal segmentations for many common potentials and the condition allows $\mathbf{b}_k$ to vary for different $k$. We defer its proof to the Appendix.

In a nutshell, this OSSVM model tries to maximize the ``margin'' from the energy corresponding to all ground-truth samples to the smallest energy produced by the same sample set, such that the margin between the energy of the positive samples and the potential energy of unknown negative samples is also maximized to some extent. We visualize this idea in Figure~\ref{fig:OSSVM}. 

\noindent {\bf Edge prior.} A price of removing the negative samples is the accuracy of the model. Without negative sample, the training data may not be sufficient to yield satisfactory MRF model. We thus propose to impose a prior during the OSSVM learning. This can be crucial to even two-class SSVM when the ground-truth data itself contains errors. Our prior is that the \emph{edge term is important to segmentation}. Thus, the weight $w^{edge}$ in Eq. (\ref{EQ:edgeterm}), which is actually $w_1$, has to be large. This is motivated by the fact that the edge cue is almost always valid, i.e. the final segmentation boundary should always adhere to image edges. To this effect, we propose to maximize the weight on edge features as much as possible. The corresponding one-class SVM model with edge prior for structured learning can be written as follows:
\begin{equation}\label{EQ:OSSVM_wp}
\begin{split}
\min_{\mathbf{w},\vec{\varepsilon}}~& {1\over2}\|\mathbf{w}\|^2 + {C\over N}\left(\sum_{k=1}^N \varepsilon_k\right) - w^{edge} \\
\hbox{s.t.:}~&  \forall k,~ \mathbf{w}\cdot\mathbf{\Psi}(\mathbf{x}_k,\mathbf{f}^*_k)\leq -1+\varepsilon_k \\            
             &  \sum_i w_i=1, \mathbf{w}\geq 0, \\
\end{split}
\end{equation}



\begin{table*}[!t]
\renewcommand{\arraystretch}{1}
\centering
\caption{Learned weights for the bilayer MRF illustrated in Fig. \ref{FIG:BilayerMRF}. {\bf A} is trained with GMM classifier. {\bf B} is trained with Zhong et al.'s classifier~\cite{Zhong2012UDC_SIGGRAPHAsia}. }\label{TB:learnedw_MRF}\vspace{-2mm}
\small{ \begin{tabular*}{\textwidth}{@{\extracolsep{\fill}}cc|ccccccccccc }
 \toprule
  \multicolumn{2}{c|}{\bf w} 				& $w_1$ & $w_2$ & $w_3$ & $w_4$ & $w_5$ & $w_6$ & $w_7$ & $w_8$ & $w_9$ & $w_{10}$ & $w_{11}$ \\ \hline 
\multirow{3}{*}{\bf A}	& 2CSSVM & 0.52&0.01&0.01&0.01&0.01&0.01&0.048&0.089&0.098&0.098&0.093\\
  & OSSVM w/o prior  & 0.00068 &   0.0059 &  0.0058  & 0.0059 &  0.0059 &  0.0059  &   0.24 &    0.14   &   0.2   &  0.18   &  0.219 \\
  & OSSVM with prior &  0.091  &  0.006 &  0.0059  &  0.006  &  0.006  &  0.006  &   0.25   &  0.12   &  0.17  &   0.14  &   0.19\\ \hline
  \multirow{3}{*}{\bf B}	& 2CSSVM &0.170&0.0233&0.0424&0.0385&0.0380&0.045&0.115 &0.134&0.130&0.129& 0.136\\  
 	& OSSVM w/o prior  & 0.0004 &	0.017 &	0.016 &	0.016 &	0.016 & 0.016 & 0.5  & 0.083 & 0.11 & 0.12 & 0.099 \\
 & OSSVM with prior & 0.091  &	0.016 &	0.015 &	0.016 &	0.016 & 0.015 &	0.44 & 0.079 & 0.11 & 0.12 & 0.093\\ 
  \bottomrule  
 \end{tabular*}}\vspace{-0.2cm}
\begin{flushleft}
{\scriptsize where $\mathbf{w} = [w_1,w_2,...,w_{11}]^T = [w^{edge},~  w^{inside}_{p,q_0},~  w^{inside}_{p,q_1},~  w^{inside}_{p,q_2},~  w^{inside}_{p,q_3},~  w^{inside}_{p,q_4},~  w^{outside}_{p,q_0},~  w^{outside}_{p,q_1},~  w^{outside}_{p,q_2},~  w^{outside}_{p,q_3},~  w^{outside}_{p,q_4}]^T$}
\end{flushleft}
\vspace{-0.5cm}
\def\arraystretch{1}
\end{table*}
                                
\subsection{Learning algorithms}
Both of the SSVM and the OSSVM can be used for learning the weights in our model. We adopt the cutting plane algorithm for two-class SSVM \cite{Joachims2009CuttingPlaneSSVM}. We include the pseudocode for this method in Algorithm \ref{ALG:2CSSVM} for self-containedness. The pseudocode for our OSSVM learning method is presented in Algorithm \ref{ALG:OSSVM}. After obtaining the weights $\mathbf{w}^*$, we can use them in the rectification model shown in Eq. (\ref{EQ:Bi-MRF}) or Eq. (\ref{EQ:Bi-MRF_linear}). 

\begin{algorithm}[!t]
	\SetKwInOut{Input}{Input}\SetKwInOut{Output}{Output}
	\LinesNumbered
	
	\Input{Training images $\{\mathbf{x}_k|k=1,2,...,N\}$ paired with predicted masks $\{h_k|k=1,2,...,N\}$ from any classifier and	ground truth masks $\{f^*_k|k=1,2,...,N\}$}
	\Output{Learned weights $\mathbf{w}^*$ for the MRF potentials}
	$\mathbf{w}^0 \leftarrow \vec{1}$\;
	$\mathcal{W}\leftarrow\emptyset$\;
	\ForAll{Images}{
		$	
		\mathbf{\Psi}^*_k \leftarrow \mathbf{\Psi}(f^*_k|\mathbf{x}_k,h_k)
		$;\textbackslash{}\textbackslash{}\texttt{By Eq. (\ref{EQ:DefinePsi})}
	}
	\While{Not Converged}{		
		\ForAll{Images}{	
			$f_k \leftarrow \min\limits_{f_k'} \Delta(f_k',f^*_k)+\mathbf{w}^j\mathbf{\Psi}(f_k'|\mathbf{x}_k,h_k)$\;			
			$	
			\mathbf{\Psi}_k \leftarrow \mathbf{\Psi}(f_k|\mathbf{x}_k,h_k)
			$\;
			$_\Delta\mathbf{\Psi}_k \leftarrow \mathbf{\Psi}^j_k-\mathbf{\Psi}^*_k$\;
			$\Delta_k \leftarrow  \Delta(f_k,f^*_k)$;\hspace{4.5cm}	
			\textbackslash{}\textbackslash{}\texttt{Cutting plane generation}\\
		}
		$\mathcal{W} \leftarrow \mathcal{W} \bigcup \{_\Delta\mathbf{\Psi}_k,\Delta_k |k=1,...,N\}$\;
		$	
		\mathbf{w}^{j+1} \leftarrow \left\{\begin{array}{ll}
		&\min_{\mathbf{w},\vec{\xi}}~ {1\over2}\|\mathbf{w}\|^2 + {C\over |\mathcal{W}|}\sum_{n=1}^{|\mathcal{W}|} {\xi}_k \\
		&\hbox{s.t.:}~  \forall \{_\Delta\mathbf{\Psi}_n,\Delta_n\}\in\mathcal{W},\\ &\hspace{1.5cm}\mathbf{w}\cdot_\Delta\mathbf{\Psi}_n\geq \Delta_n - \xi_n, \\
		&\hspace{0.7cm}\sum_i w_i=1, \mathbf{w}\geq 0
		\end{array}\right.$;
		
		\textbackslash{}\textbackslash{}\texttt{$|\cdot|$ denotes number of elements}\\
		$j\leftarrow j+1$\;	
	}
	$\mathbf{w}^*\leftarrow\mathbf{w}^j$\;
	\caption{Two-class SSVM learning \cite{Joachims2009CuttingPlaneSSVM}}\label{ALG:2CSSVM}
\end{algorithm}

\begin{algorithm}
	\SetKwInOut{Input}{Input}\SetKwInOut{Output}{Output}
	\LinesNumbered
	
	\Input{Same as in Algorithm \ref{ALG:2CSSVM}}
	\Output{Same as in Algorithm \ref{ALG:2CSSVM}}
	\ForAll{Images}{
		$	
		\mathbf{\Psi}^*_k \leftarrow \mathbf{\Psi}(f^*_k|\mathbf{x}_k,h_k)
		$;\textbackslash{}\textbackslash{}\texttt{By Eq. (\ref{EQ:DefinePsi})}
	}

	$
	\mathbf{w}^*\leftarrow\left\{\begin{array}{ll}
	\min\limits_{\mathbf{w},\vec{\varepsilon}}~& {1\over2}\|\mathbf{w}\|^2 + {C\over N}\left(\sum_{k=1}^N \varepsilon_k\right) - w^{edge} \\
	\hbox{s.t.:}~&  \forall k,~ \mathbf{w}\cdot\mathbf{\Psi}^*_k\leq -1+\varepsilon_k, \\            
	&  \sum_i w_i=1, \mathbf{w}\geq 0 \\
	\end{array}\right.;
	$
	\textbackslash{}\textbackslash{}\texttt{According to Eq. (\ref{EQ:OSSVM_wp})}\\
	\caption{OSSVM learning}\label{ALG:OSSVM}
\end{algorithm}

\section{Practical concerns in implementation}\label{sec:practical}

\subsection{The shrinking bias of graph cut}

It is well known that graph cut for MRF model with 2nd-order pairwise potential suffers from the shrinking bias \cite{Veksler08StarShapePrior,Brian10GeoGraphCut}.  
The recently proposed local forground-background classifiers~\cite{Bai09VideoSnapCut_SIGGRAPH,Zhong2012UDC_SIGGRAPHAsia} are shown to be able to correct local errors near the object boundary. The shrinking bias falls into this category as it introduces small errors near the boundary. Hence, we propose to feed the results of the rectified segmentation to re-train the foreground-background classifiers on the current frame, and use the updated classifiers to perform classification in the same frame again, to avoid the shrinking bias. 

\subsection{Computational complexity}
There exist quite a few efficient algorithms for solving graph cuts, i.e. the max-flow/min-cut problem. The computational time for the Boykov-Kolmogorov (BK) algorithm on a 2 MP image on CPU is about 160 ms, and the GPU implementation of graph cuts can be 2 times faster than on the CPU \cite{vineet2008cuda}. The foreground-background classifier we use is the one reported in \cite{Zhong2012UDC_SIGGRAPHAsia}, and its average computational time is about 1.5s for one frame on a PC with quad-core 3.3 GHz CPUs. Optical flows and edge maps can all be precomputed.

%% file: results_CVPR_revised.tex
 worth showing due to the page limit. We include them in the supplementary material.
\section{Experimental results}\label{sec:experiment}
In the experiments, we compare our method with state-of-the-art methods for full sequence cutout with the initial keyframe segmentation. We avoid end-to-end system comparison since the interactive segmentation phase, i.e., Step 1, in the video cutout system is out of the focus of this work. We are unable to present all the results worth showing due to the page limit. We include them in the supplementary material.


\subsection{Experimental Setup}
\noindent {\bf Datasets.} 
We mainly evaluate our method on the dataset proposed in Zhong et al.~\cite{Zhong2012UDC_SIGGRAPHAsia}. The main advantage of this dataset is that the ground truth segmentation for each frame has been provided. The training data from Zhong et al.'s dataset contains 15 video sequences in total, 9 for training and 6 for testing. The \emph{Training set} we used in this work contains 2012 frames from their 9 training sequences, and the \emph{Test set} contains 741 frames from their testing sequences.

\noindent {\bf Learning the weights.}
We use the training set to learn the parameters $\{w^{inside}_{pq},w^{outside}_{pq},w^{edge}\}$ of the bilayer MRF with both the conventional SSVM and our one-class SSVM as presented in section~\ref{SEC:OSSVM}. 

We considered applying our framework to rectifying the output of two classifiers. One is the Gaussian mixture model (GMM), which is a typical global foreground-background (FB) classifier, and the other is the state-of-the-art local FB classifier proposed in~\cite{Zhong2012UDC_SIGGRAPHAsia}. We apply the FB classifiers to all the consecutive frame pairs to generate the classifier output to be rectified. The classifier outputs, the ground truth segmentations, together with the images are then fed into the structured learning framework. Table~\ref{TB:learnedw_MRF} shows the trained parameters for both GMM and Zhong et al.'s classifier using different learning models. We empirically chose maximum iteration number to be 10 for the training process. The training time for OSSVM is about 0.062 seconds in MATLAB on Intel® Core™ i7-4700MQ Processor, while the training time for the conventional two class SSVM is about 7500 seconds (10 iteractions). 

It is interesting to note that the weights show \emph{strong asymmetry} structure, and the learned weights would penalize more false positive than false negative. Besides, the weights for GMM are more uniform for both \emph{inside} and \emph{outside} weights. This means the results by GMM is very noisy and strong smoothness is required for rectifying the GMM classifier. 
\begin{figure}[t]
  \centering
  \includegraphics[width=0.95\columnwidth]{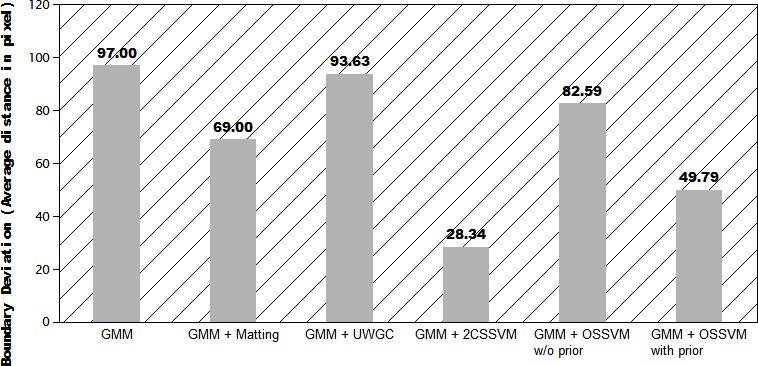}\\
  \caption{Effectiveness of our segmentation rectification for GMM classifier.}\label{FIG:RstRectfyGMM}
  \vspace{-0.5cm}
  \end{figure}
  
\noindent {\bf Evaluation metrics.} We mainly use \emph{boundary deviation} to measure the segmentation performance. It is defined as the average distance from the estimated boundary to the ground truth boundary. 

\noindent {\bf Methods for evaluation.} We mainly evaluate two variations of our method in the experiments: FB classifier + graph cuts with weights learned by two class structured SVM (2CSSVM), and FB classifier + graph cuts with weights learned by one class structured SVM (OSSVM). We shall call them: \textbf{Our method (2CSSVM)}, \textbf{Our method (OSSVM)}. We applied our method to GMM based FB classifier and the state-of-the-art FB classifier proposed in \cite{Zhong2012UDC_SIGGRAPHAsia}. We main compare with the FB classifier + Matting, as adopted in \cite{Zhong2012UDC_SIGGRAPHAsia}, and FB classifier + uniformly weighted graph cuts (UWGC), which is adopted in Rotobrush~\cite{Bai09VideoSnapCut_SIGGRAPH}.

\subsection{Rectifying global classifier}

\begin{figure}[!t]
	\begin{center}
		\begin{tabular}{
				@{\hspace{0mm}}c@{\hspace{0mm}}c@{\hspace{0mm}}c @{\hspace{0mm}}c
				@{\hspace{0mm}}c@{\hspace{0mm}}c@{\hspace{0mm}}c @{\hspace{0mm}}c
				@{\hspace{0mm}}c@{\hspace{0mm}}c
			}
			\begin{sideways}\parbox{20mm}{\centering\scriptsize GMM classifier}\end{sideways} &
			\includegraphics[width=0.4\columnwidth]{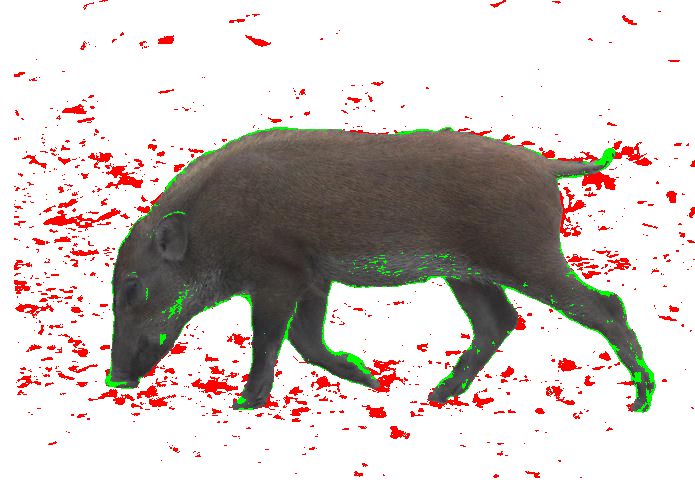}&
			\includegraphics[width=0.4\columnwidth]{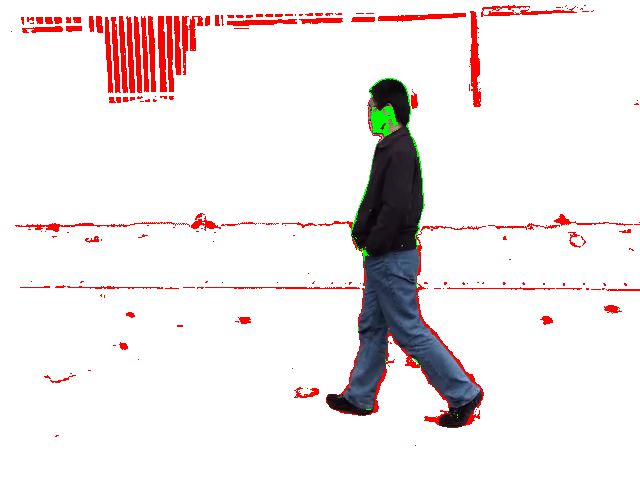}\\
			\begin{sideways}\parbox{20mm}{\centering\scriptsize GMM + Matting}\end{sideways} &
			\includegraphics[width=0.4\columnwidth]{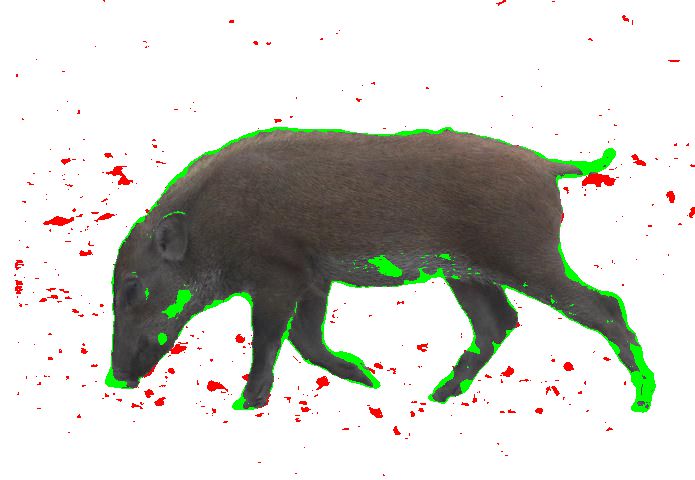}&
			\includegraphics[width=0.4\columnwidth]{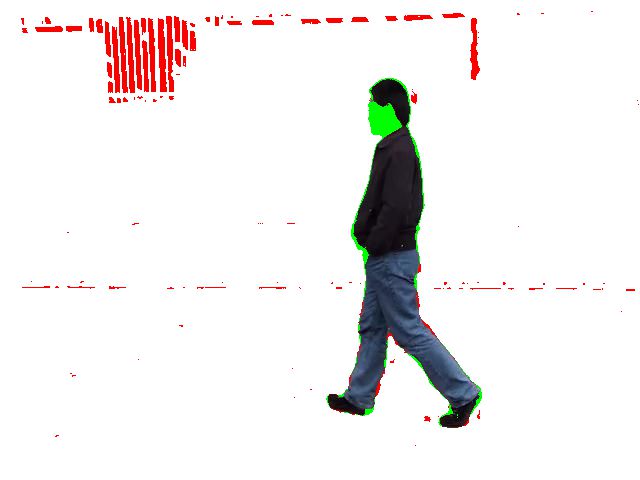}\\
			\begin{sideways}\parbox{20mm}{\centering\scriptsize GMM + GC}\end{sideways} &
			\includegraphics[width=0.4\columnwidth]{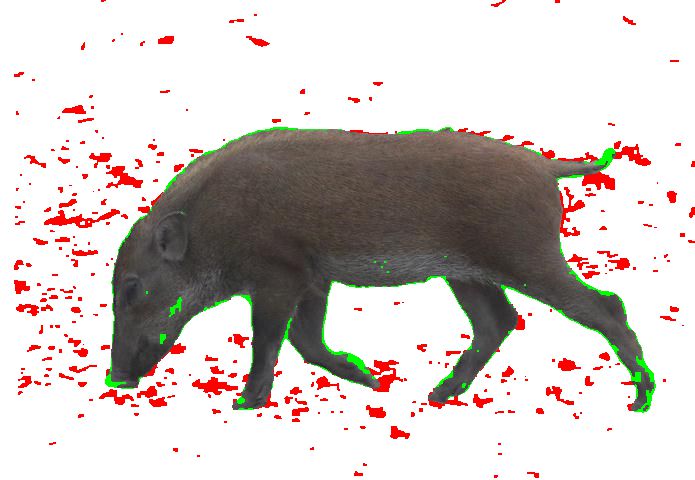}&
			\includegraphics[width=0.4\columnwidth]{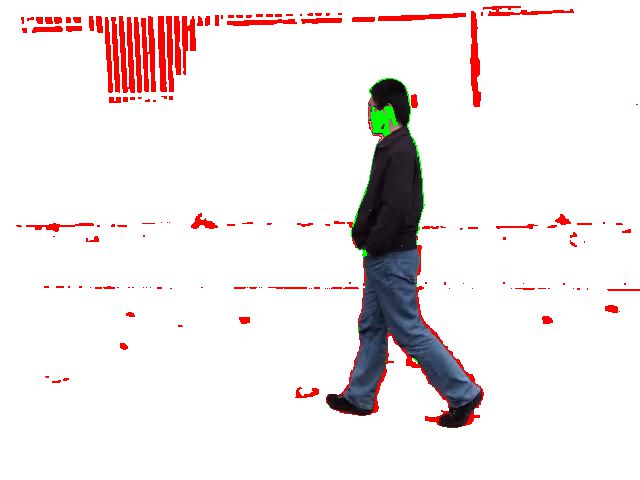}\\
			\begin{sideways}\parbox{20mm}{\centering\scriptsize GMM + 2CSSVM}\end{sideways} &
			\includegraphics[width=0.4\columnwidth]{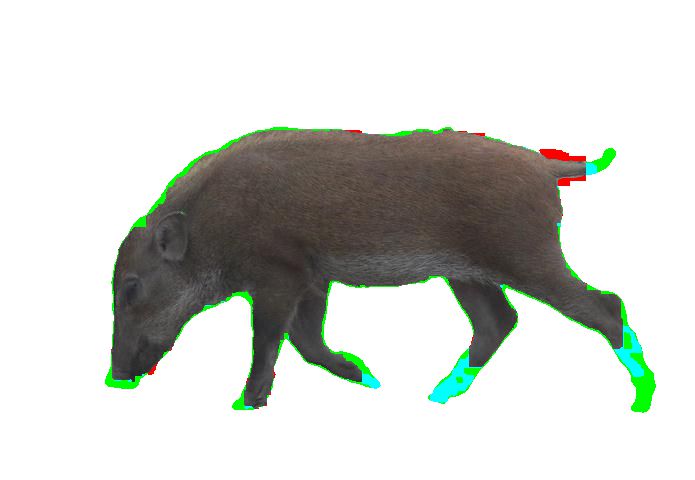}&
			\includegraphics[width=0.4\columnwidth]{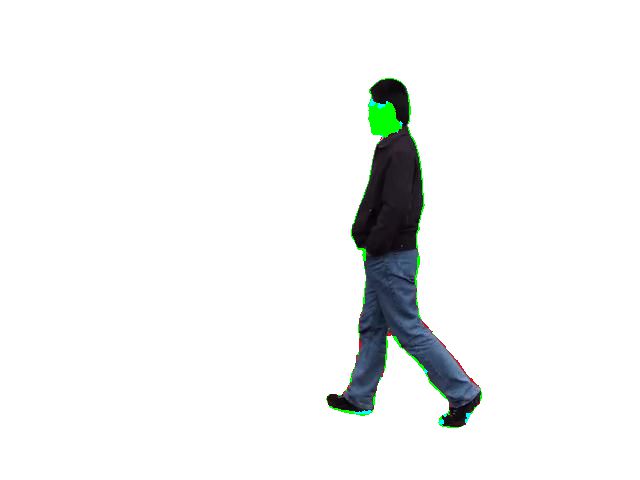}\\
			\begin{sideways}\parbox{20mm}{\centering\scriptsize GMM + OSSVM (w/o prior)}\end{sideways} &
			\includegraphics[width=0.4\columnwidth]{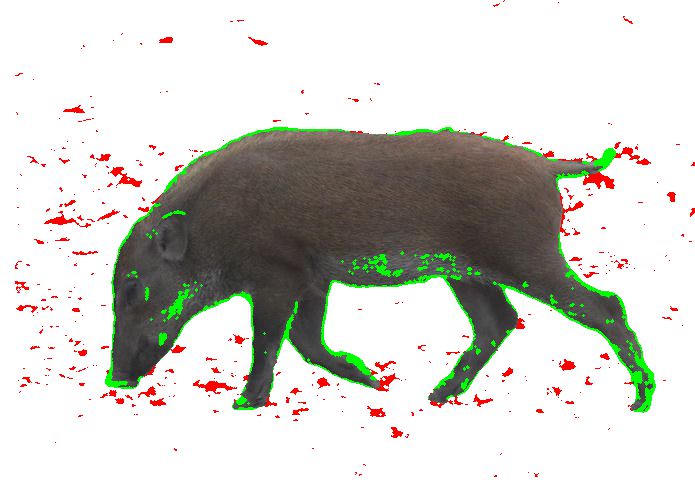}&
			\includegraphics[width=0.4\columnwidth]{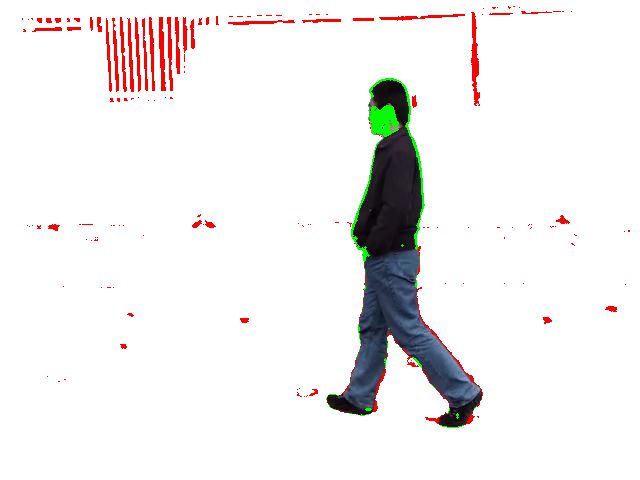}\\
			\begin{sideways}\parbox{20mm}{\centering\scriptsize GMM + OSSVM (with prior)}\end{sideways} &
			\includegraphics[width=0.4\columnwidth]{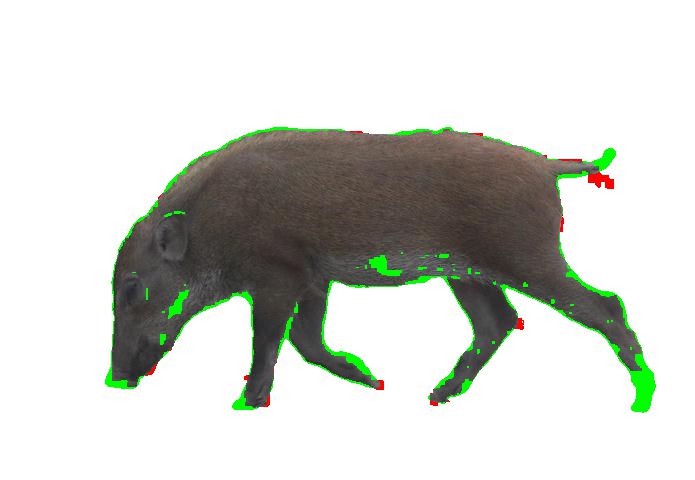}&
			\includegraphics[width=0.4\columnwidth]{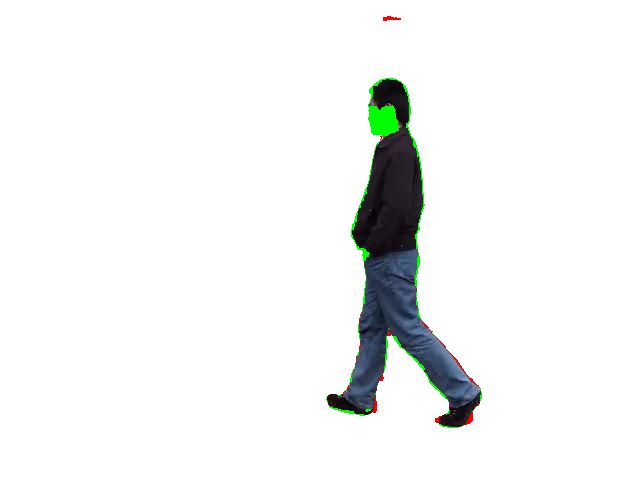}\\
		\end{tabular}
	\end{center}\vspace{-0.2cm}
	\caption{Comparison of our approach with GMM classifier + other refinement methods. The green regions are the missing foreground regions (FN) and the red regions are the unwanted background regions (FP).}\label{FIG:ExampleGMM}\vspace{-0.2cm}
\end{figure}

\begin{figure*}[!htb]
	\centering
	\begin{tabular}{
			@{\hspace{0mm}}c@{\hspace{0mm}}c@{\hspace{0mm}}c @{\hspace{0mm}}c
			@{\hspace{0mm}}c@{\hspace{0mm}}c@{\hspace{0mm}}c @{\hspace{0mm}}c
			@{\hspace{0mm}}c@{\hspace{0mm}}c
		}
		\includegraphics[height = 1.8cm]{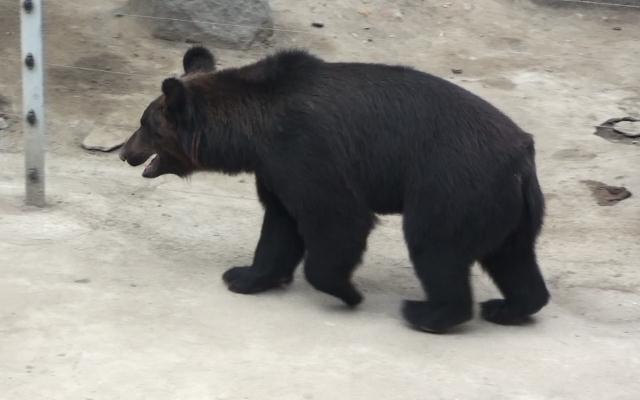}&
		\includegraphics[height = 1.8cm]{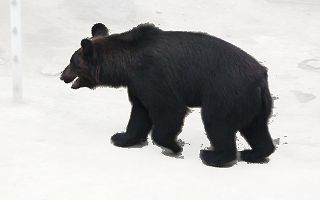}&
		\includegraphics[height = 1.8cm]{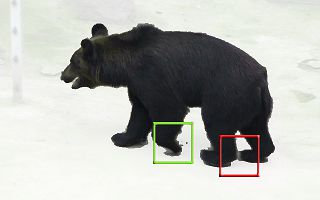}&
		\includegraphics[height = 1.8cm]{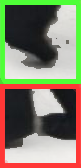}&
		\includegraphics[height = 1.8cm]{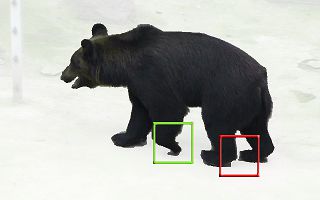}&
		\includegraphics[height = 1.8cm]{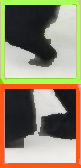}&
		\includegraphics[height = 1.8cm]{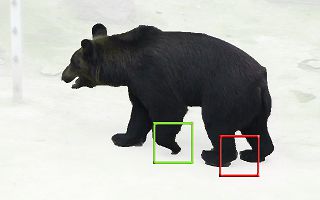}&
		\includegraphics[height = 1.8cm]{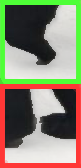}\\
		\includegraphics[height = 1.8cm]{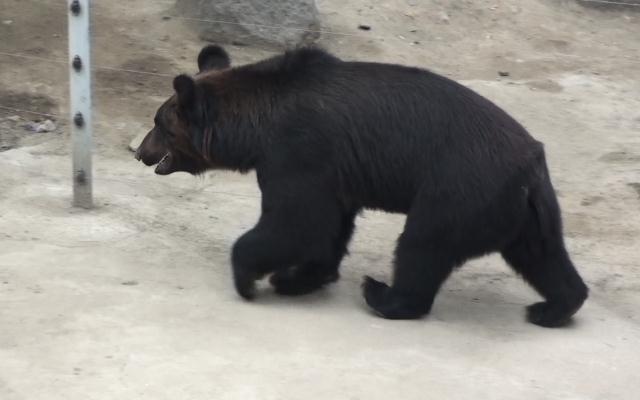}&
		\includegraphics[height = 1.8cm]{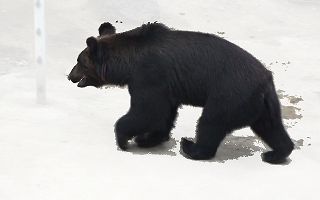}&
		\includegraphics[height = 1.8cm]{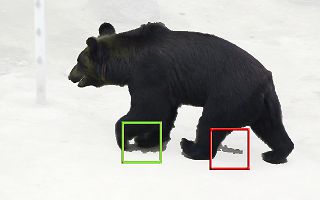}&
		\includegraphics[height = 1.8cm]{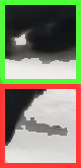}&
		\includegraphics[height = 1.8cm]{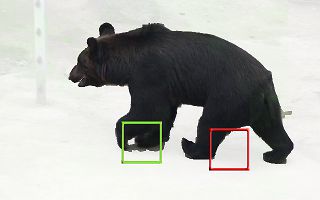}&
		\includegraphics[height = 1.8cm]{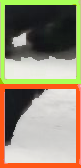}&          
		\includegraphics[height = 1.8cm]{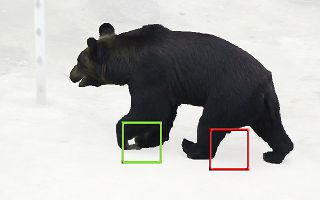}&
		\includegraphics[height = 1.8cm]{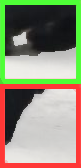}\\
		{\scriptsize Original image} & {\scriptsize FB Classifier \cite{Zhong2012UDC_SIGGRAPHAsia} + Matting} & {\scriptsize FB Classifier + UWGC} &{\scriptsize Zoom in} & {\scriptsize Our method (2CSSVM)} &{\scriptsize Zoom in} &{\scriptsize Our method (OSSVM)}& {\scriptsize Zoom in}\\\vspace{-0.8cm}
	\end{tabular}          
	\caption{Results of fully automatic segmentation propagation for the ``Bear'' sequence on frames 5 (top) and 9 (bottom), given the same keyframe segmentation. The background is whitened for visualization.}\label{FIG:ExampleSeq2}\vspace{-0.2cm}
\end{figure*}

\begin{figure*}[!htb]
	\centering
	\begin{tabular}{
			@{\hspace{0mm}}c@{\hspace{0mm}}c@{\hspace{0mm}}c @{\hspace{0mm}}c
			@{\hspace{0mm}}c@{\hspace{0mm}}c@{\hspace{0mm}}c @{\hspace{0mm}}c
			@{\hspace{0mm}}c@{\hspace{0mm}}c
		}
		\includegraphics[height=1.65cm]{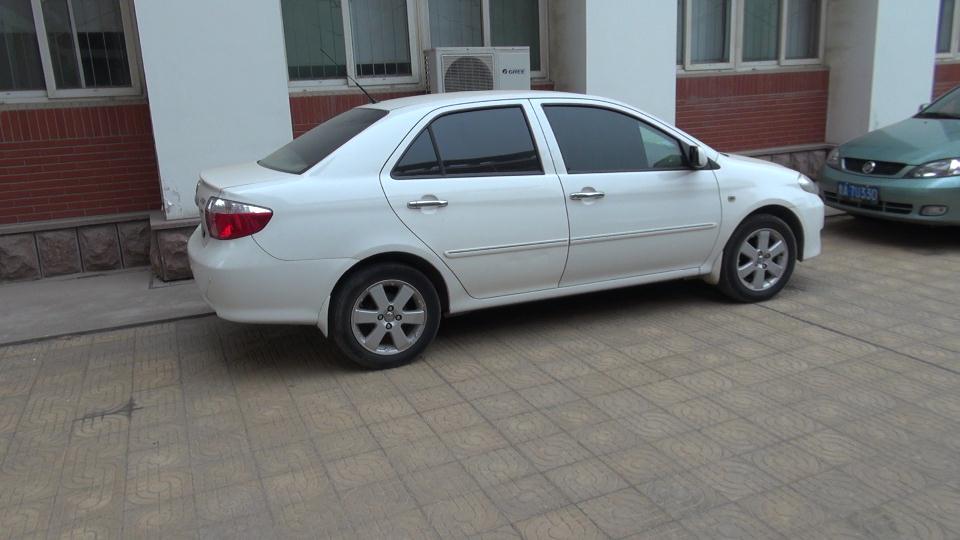}&
		\includegraphics[height=1.65cm]{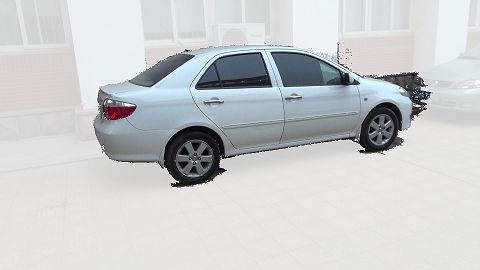}&
		\includegraphics[height=1.65cm]{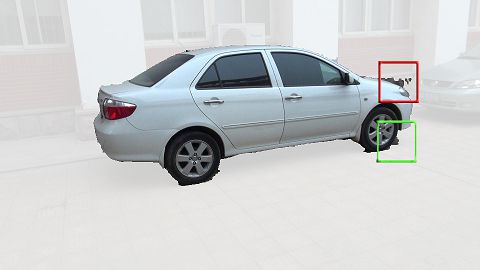}&          
		\includegraphics[height=1.65cm]{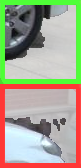}&
		\includegraphics[height=1.65cm]{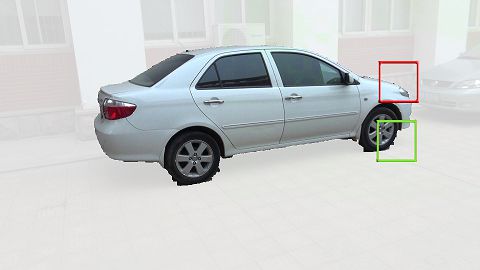}&          
		\includegraphics[height=1.65cm]{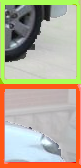}&
		\includegraphics[height=1.65cm]{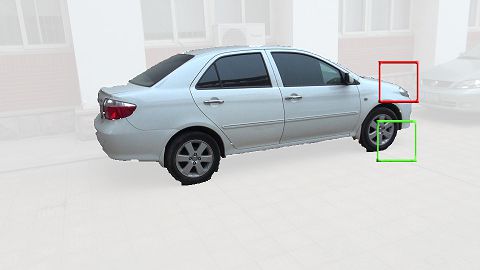}&
		\includegraphics[height=1.65cm]{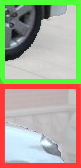}\\
		\includegraphics[height=1.65cm]{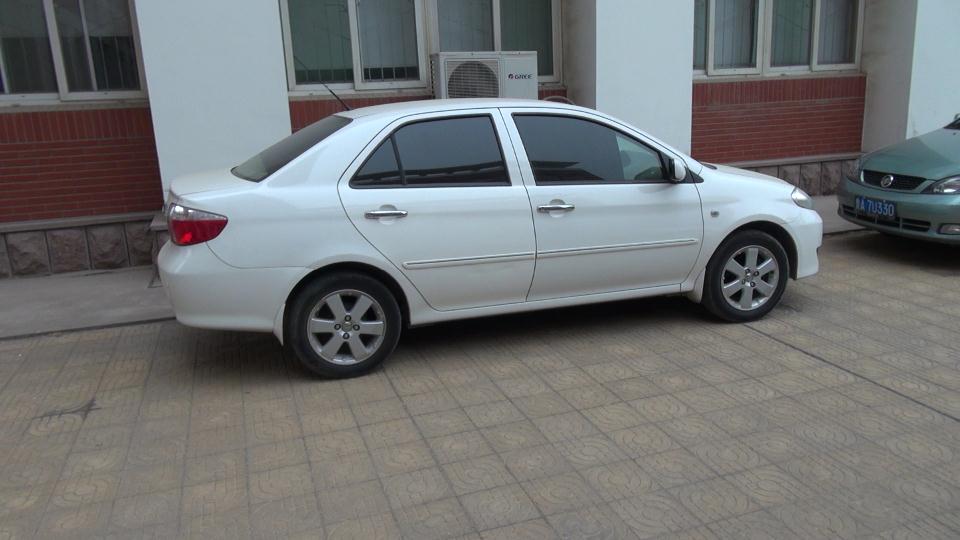}&
		\includegraphics[height=1.65cm]{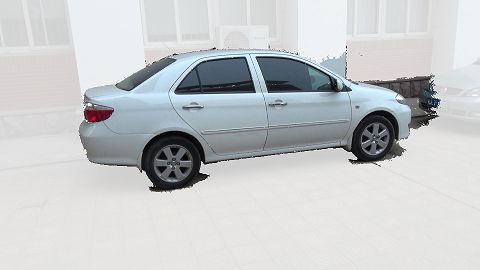}&
		\includegraphics[height=1.65cm]{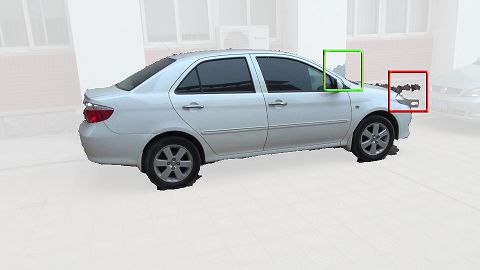}&
		\includegraphics[height=1.65cm]{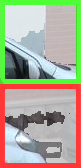}&          
		\includegraphics[height=1.65cm]{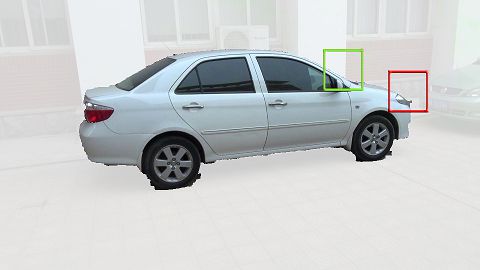}&
		\includegraphics[height=1.65cm]{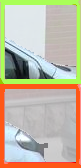}&          
		\includegraphics[height=1.65cm]{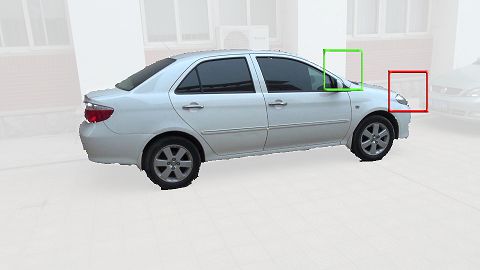}&
		\includegraphics[height=1.65cm]{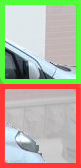}\\
		{\scriptsize Original image} & {\scriptsize FB Classifier~\cite{Zhong2012UDC_SIGGRAPHAsia} + Matting} & {\scriptsize FB Classifier + UWGC} &{\scriptsize Zoom in} &{\scriptsize Our method (2CSSVM)}& {\scriptsize Zoom in} &{\scriptsize Our method (OSSVM)}& {\scriptsize Zoom in}\\\vspace{-0.8cm}
	\end{tabular}          
	\caption{Results of fully automatic segmentation propagation for the ``Car'' sequence on frames 22 (top) and 24 (bottom), given the same keyframe segmentation. The background is whitened for visualization.}\label{FIG:ExampleSeq5}
\end{figure*}

We first apply our framework to the global Gaussian Mixture Model (GMM) classifier for segmentation rectification. Generally, a global classifier like GMM is not suitable to video cutout~\cite{Bai09VideoSnapCut_SIGGRAPH}, we conduct the experiment to validate the generality of our framework.

In the segmentation propagation, we train a GMM classifier using the ground truth segmentation on frame $t$, and apply it on frame $t+1$ as the propagated segmentation, and then apply our rectification approach with different weights, e.g., learned with or without prior as shown in Table~\ref{TB:learnedw_MRF}, for further refinement. The average boundary deviation are summarized in Figure~\ref{FIG:RstRectfyGMM}. Some typical results are shown in Figure \ref{FIG:ExampleGMM}. The results suggest that our rectification approach with learned weights can significantly improve the segmentation results generated by the GMM classifier. 

Due to the well-known non-local behavior of GMM based classifier, it's not surprising to see that \textbf{2CSSVM} and \textbf{OSSVM with prior} significantly outperform all other cases. The experimental results shown in this subsection also imply that the errors from GMM for image cutout can be more effectively removed by using our method, compared with the conventional methods based on graph cuts and matting~\cite{Rother04GrabCut,BoykovJolly01GMM-MRF}, and the 2CSSVM outperforms the rest in this case while the OSSVM is the best alternative when computational efficiency in learning is a concern. Note that as shown in the visual results in Figure \ref{FIG:ExampleGMM}, significant errors can still be observed in the refined output, and this reasserts that global classifier such as GMM alone is not suitable for video cutout. 

It is also interesting to note that global classifiers, such as GMM, are commonly adopted in image cutout~\cite{Rother04GrabCut,BoykovJolly01GMM-MRF}. Hence, these results suggests that our method may be applied to image cutout as well.



\begin{figure*}[!t]
	\centering
	\begin{tabular}{rl}
		\vspace{-0.5cm}
		\begin{sideways}\parbox{27mm}{\centering \scriptsize FB Classifier~\cite{Zhong2012UDC_SIGGRAPHAsia} + Matting}\end{sideways}&\hspace{-0.3cm}
		\subfloat{ 
			\includegraphics[height=2.7cm]{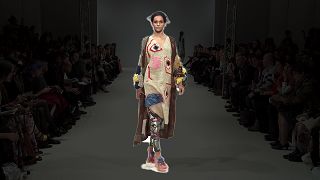}
			\includegraphics[height=2.7cm]{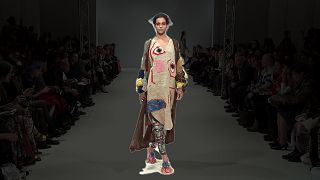}
			\includegraphics[height=2.7cm]{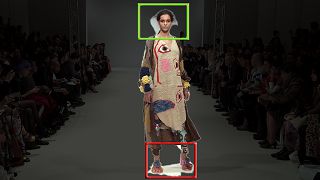}
			\includegraphics[height=2.7cm]{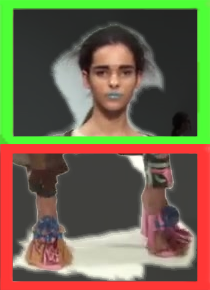}}\\ \vspace{-0.5cm}
		\begin{sideways}\parbox{27mm}{\centering \scriptsize FB Classifier + UWGC}\end{sideways}&\hspace{-0.3cm}
		\subfloat{ 
			\includegraphics[height=2.7cm]{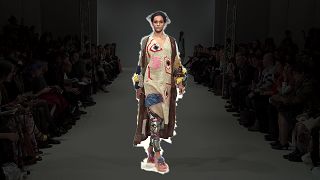}
			\includegraphics[height=2.7cm]{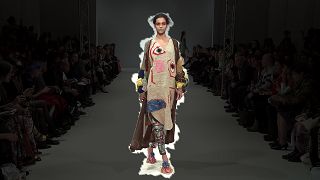}
			\includegraphics[height=2.7cm]{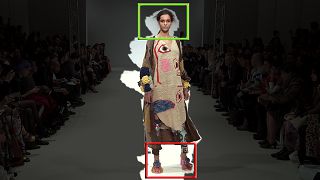}
			\includegraphics[height=2.7cm]{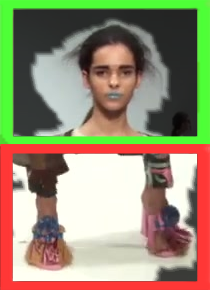}}\\ \vspace{-0.5cm}
		\begin{sideways}\parbox{27mm}{\centering  \scriptsize Adobe Rotobrush}\end{sideways}&\hspace{-0.3cm}
		\subfloat{ 
			\includegraphics[height=2.7cm]{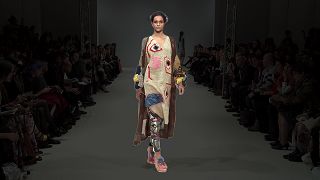}
			\includegraphics[height=2.7cm]{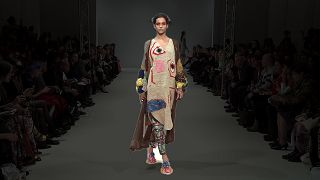}
			\includegraphics[height=2.7cm]{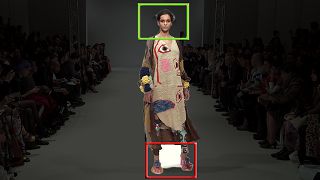}
			\includegraphics[height=2.7cm]{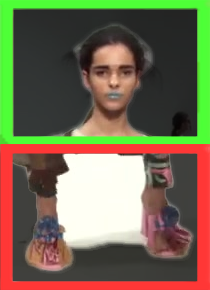}}\\\vspace{-0.5cm}
		\begin{sideways}\parbox{27mm}{\centering \scriptsize Our method (2CSSVM)}\end{sideways}&\hspace{-0.3cm}
		\subfloat{ 
			\includegraphics[height=2.7cm]{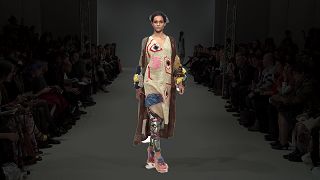}
			\includegraphics[height=2.7cm]{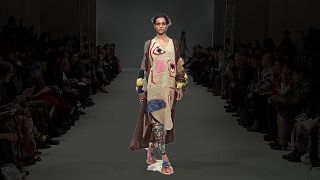}
			\includegraphics[height=2.7cm]{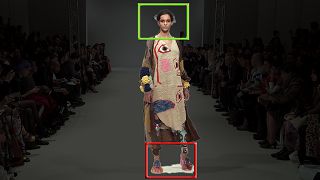}
			\includegraphics[height=2.7cm]{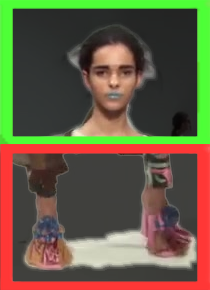}}\\
		\begin{sideways}\parbox{27mm}{\centering \scriptsize Our method (OSSVM)}\end{sideways}&\hspace{-0.3cm}
		\subfloat{ 
			\includegraphics[height=2.7cm]{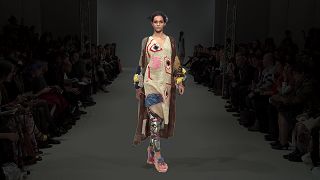}
			\includegraphics[height=2.7cm]{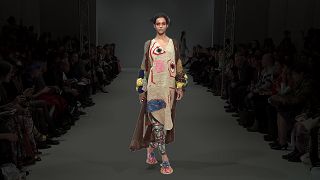}
			\includegraphics[height=2.7cm]{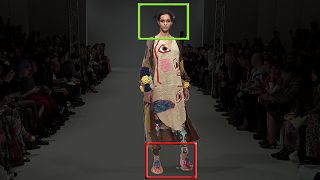}
			\includegraphics[height=2.7cm]{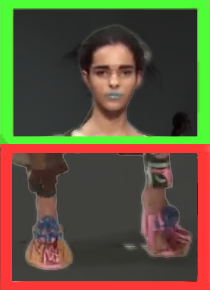}}\\ 
	\end{tabular}\vspace{-0.3cm}          
	\caption{Results of fully automatic segmentation propagation for the ``Catwalk'' sequence on frames 1 (left), 6 (middle) and 19 (right), given keyframe segmentation. The background is darkened for visualization.}\label{FIG:ExampleSeq7}
\end{figure*}

\begin{figure*}[!t]
	\centering
	\begin{tabular}{rl}
		\vspace{-0.5cm}
		\begin{sideways}\parbox{28mm}{\centering\scriptsize FB Classifier~\cite{Zhong2012UDC_SIGGRAPHAsia} + Matting }\end{sideways}&\hspace{-0.3cm}
		\subfloat{ 
			\includegraphics[height=2.75cm]{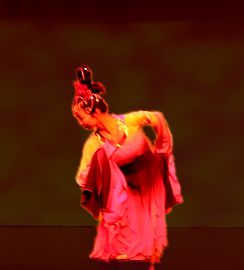}
			\includegraphics[height=2.75cm]{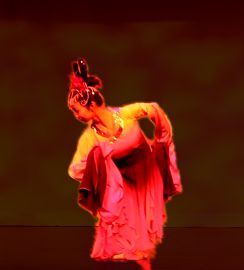}
			\includegraphics[height=2.75cm]{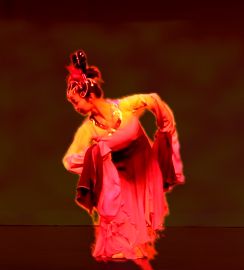}
			\includegraphics[height=2.75cm]{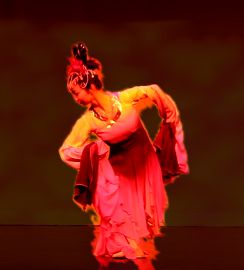}
			\includegraphics[height=2.75cm]{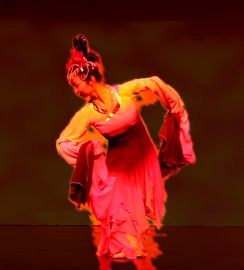}
			\includegraphics[height=2.75cm]{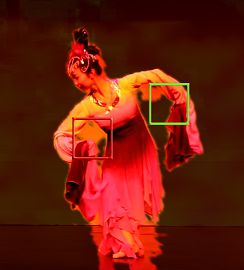}
			\includegraphics[height=2.75cm]{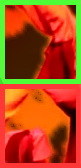}}\\ \vspace{-0.5cm}
		\begin{sideways}\parbox{25mm}{\centering\scriptsize FB-C + UWGC}\end{sideways}&\hspace{-0.3cm}
		\subfloat{ 
			\includegraphics[height=2.75cm]{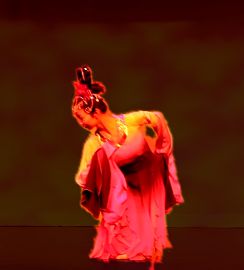}
			\includegraphics[height=2.75cm]{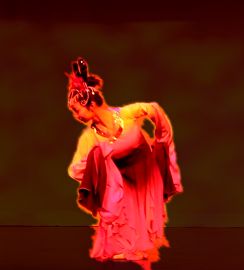}
			\includegraphics[height=2.75cm]{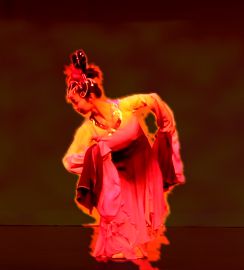}
			\includegraphics[height=2.75cm]{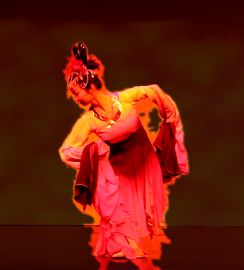}
			\includegraphics[height=2.75cm]{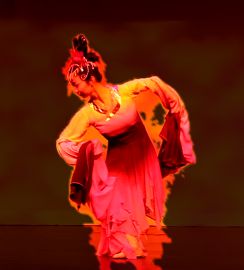}
			\includegraphics[height=2.75cm]{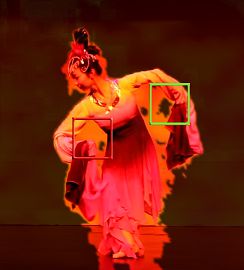}
			\includegraphics[height=2.75cm]{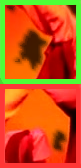}}\\ \vspace{-0.5cm}
		\begin{sideways}\parbox{25mm}{\centering \scriptsize Adobe Rotobrush}\end{sideways}&\hspace{-0.3cm}
		\subfloat{ 
			\includegraphics[height=2.75cm]{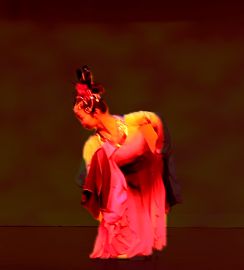}
			\includegraphics[height=2.75cm]{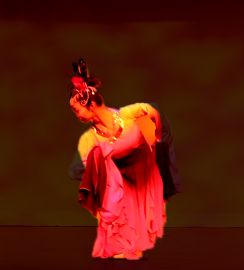}
			\includegraphics[height=2.75cm]{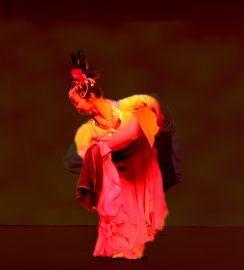}
			\includegraphics[height=2.75cm]{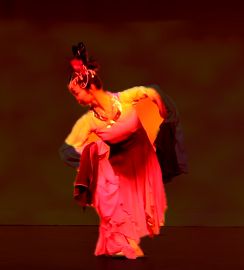}
			\includegraphics[height=2.75cm]{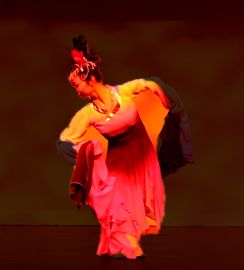}
			\includegraphics[height=2.75cm]{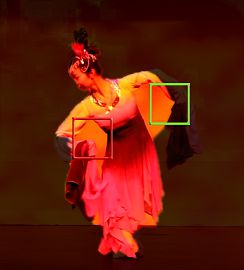}
			\includegraphics[height=2.75cm]{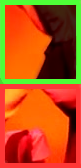}}
		\\\vspace{-0.5cm}
		\begin{sideways}\parbox{25mm}{\centering\scriptsize Our method (2CSSVM)}\end{sideways}&\hspace{-0.3cm}
		\subfloat{ 
			\includegraphics[height=2.75cm]{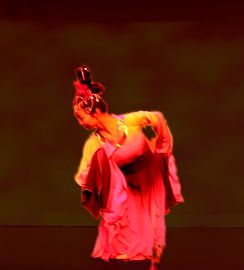}
			\includegraphics[height=2.75cm]{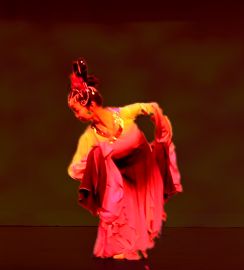}
			\includegraphics[height=2.75cm]{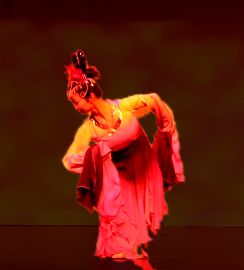}
			\includegraphics[height=2.75cm]{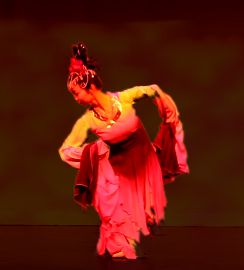}
			\includegraphics[height=2.75cm]{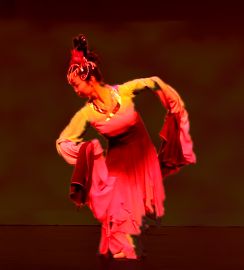}
			\includegraphics[height=2.75cm]{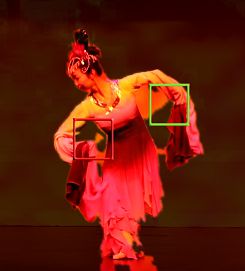}
			\includegraphics[height=2.75cm]{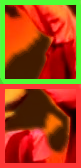}}\\
		\begin{sideways}\parbox{25mm}{\centering\scriptsize Our method (OSSVM)}\end{sideways}&\hspace{-0.3cm}
		\subfloat{ 
			\includegraphics[height=2.75cm]{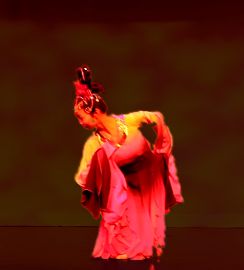}
			\includegraphics[height=2.75cm]{imgs/ChineseDance/output_Imat_learnedw_reclassify/FemaleSoloDancePhoenixHD_0833.jpg}
			\includegraphics[height=2.75cm]{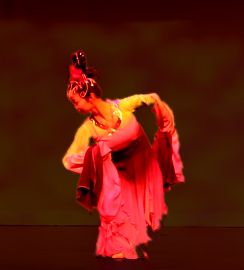}
			\includegraphics[height=2.75cm]{imgs/ChineseDance/output_Imat_learnedw_reclassify/FemaleSoloDancePhoenixHD_0835.jpg}
			\includegraphics[height=2.75cm]{imgs/ChineseDance/output_Imat_learnedw_reclassify/FemaleSoloDancePhoenixHD_0836.jpg}
			\includegraphics[height=2.75cm]{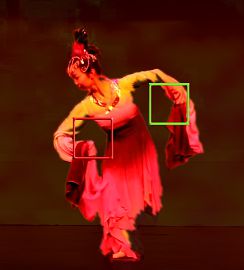}
			\includegraphics[height=2.75cm]{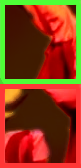}}\\ 
	\end{tabular}\vspace{-0.2cm} 
	\caption{Results of automatic segmentation propagation for the ``Chinese dance'' sequence on frames 32 to 37 (left-right), given the same keyframe segmentation. The background is darkened for visualization.}\label{FIG:ExampleSeq6}
\end{figure*}

\subsection{Rectifying local classifier}

We further demonstrate the efficacy of our approach for rectifying the state-of-the-art local classifier proposed in~\cite{Zhong2012UDC_SIGGRAPHAsia}. We perform the full sequence propagation on the {Test set} and Figure~\ref{FIG:Rst_SeqPropag2} shows the quantitative results. Again, we see the segmentation errors by matting and graph cuts with uniform weights accumulates more rapidly than ours. After 10 frames of propagation, matting and graph cut generates 4 times and 2 times more error than our method respectively. We also present the visual results for a ``Car'' sequence in the Test set in Figure~\ref{FIG:ExampleSeq5}, and a ``Bear'' sequence from the Training set in Figure~\ref{FIG:ExampleSeq2}. 

We also compare our method with Video SnapCut (Rotobrush)~\cite{Bai09VideoSnapCut_SIGGRAPH} and~\cite{Zhong2012UDC_SIGGRAPHAsia} on additional video sequences. Some of the typical results are shown in Figures \ref{FIG:ExampleSeq7} and \ref{FIG:ExampleSeq6}. We can observe that the Rotobrush could suffer from the temporal discontinuity while segmentation error could easily accumulate in other methods based on Zhong et al.'s classifier. 
\begin{figure}[!t]
	\centering
	\includegraphics[width=0.95\columnwidth]{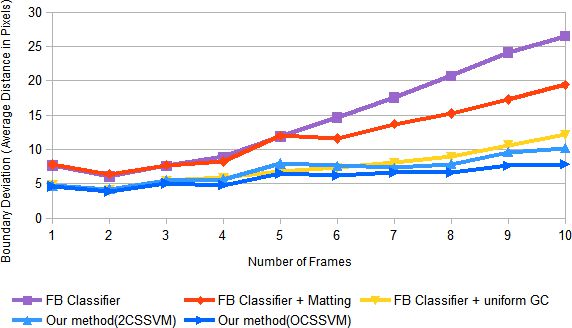}
	\caption{Error accumulation in segmentation propagation over full sequences on Test set without additional user interaction.  }\label{FIG:Rst_SeqPropag2}\vspace{-0.5cm}
\end{figure}

\subsection{Experiment on RGB-D videos}
3D movies and videos have now gained increased popularity. An extra depth channel in addition to the RGB channels may be handily available sooner or later. Since our OSSVM framework can be easily extended to handling the extra depth dimension, we also apply our method to RGB-D data for evaluation. The segmentation rectification model for RGB-D data is the same as Eq. (\ref{EQ:Bi-MRF_linear}), except that the potential is defined as
\begin{equation}\label{EQ:DefinePsiRGBD}
\begin{split}
\mathbf{\Psi}_{RGBD} = \left(\begin{array}{l}
\sum_{pp'}w^{eRGB}_{pp'} \big|f^{t}_p-{f^{t}_{p'}}\big|,\{p,p'\}\in\mathcal{N} \\
\sum_{pp'}w^{eD}_{pp'} \big|f^{t}_p-{f^{t}_{p'}}\big|,\{p,p'\}\in\mathcal{N}\\
\sum_{pq}(1-h^t_q)f^{t}_p,\{p,q\}\in\mathcal{N}_{fh} \\
\sum_{pq}h^t_q(1-f^{t}_p),\{p,q\}\in\mathcal{N}_{fh}
\end{array}\right),
\end{split}
\end{equation}
\begin{equation}\label{EQ:weqRGBD}
w^{eD}_{pp'}=\left\{\begin{array}{lr}
\exp(-5I_{eD}(p,p')),~ & I_{eD}(p,p')\neq0 \\
20,~ & \hbox{Otherwise}
\end{array}
\right.,
\end{equation}
where $I_{eD}$ is the edge map from depth channel, $w^{eRGB}_{pp'}=w^e_{pp'}$ has been defined in Eq. (\ref{EQ:weq}).

Similar to the RGB case, the OSSVM model for RGB-D is
\begin{equation}\label{EQ:OSSVM_RGBD_p}
\begin{split}
\min_{\mathbf{w},\vec{\varepsilon}}~& {1\over2}\|\mathbf{w}\|^2 + {C\over N}\left(\sum_{k=1}^N \varepsilon_k\right) - w^{edge}_{RGB} \\
\hbox{s.t.:}~&  \forall k,~ \mathbf{w}\cdot\mathbf{\Psi}_{RGBD}(\mathbf{x}_k,\mathbf{f}^*_k)\leq -1+\varepsilon_k, \\
&  w^{edge}_{RGB}\leq w^{edge}_{D}, \\            
&  \sum_i w_i=1, \mathbf{w}\geq 0, \\
\end{split}
\end{equation}
in which we added one additional constraint to the model to represent our prior that the edge term from depth is more reliable than that from RGB values and this reformulation does not require additional free parameters. 
\begin{table*}
	\renewcommand{\arraystretch}{1.2}
	\centering
	\caption{Learned weights for RGB-D data}\label{TB:learnedw_RGBD}\vspace{-2mm}
	\small{ \begin{tabular*}{\textwidth}{@{\extracolsep{\fill}}c|cc|cccccccccc }
			\toprule
			{\bf w} 				& $w_1$ & $w_2$ & $w_3$ & $w_4$ & $w_5$ & $w_6$ & $w_7$ & $w_8$ & $w_9$ & $w_{10}$ & $w_{11}$ & $w_{12}$\\ \hline 
			 2CSSVM &0.1& 0.044 &  0.084 &   0.09  &  0.087   & 0.088  &  0.089  &   0.08   & 0.085  &  0.082  &  0.083 &   0.084\\
			 OSSVM with prior = 0.5 & 0.12  &   0.19  &  0.059  &  0.071  &  0.063  &  0.069  &  0.071   & 0.064 &   0.076  &  0.074 &   0.067  &  0.077\\
			\bottomrule  
\end{tabular*}}\vspace{-0.2cm}
\begin{flushleft}
	{\scriptsize where $\mathbf{w} = [w_1,w_2,...,w_{12}]^T = [w^{edge}_{RGB},~w^{edge}_{D},~  w^{inside}_{p,q_0},~  w^{inside}_{p,q_1},~  w^{inside}_{p,q_2},~  w^{inside}_{p,q_3},~  w^{inside}_{p,q_4},~  w^{outside}_{p,q_0},~  w^{outside}_{p,q_1},~  w^{outside}_{p,q_2},~  w^{outside}_{p,q_3},~  w^{outside}_{p,q_4}]^T$}
\end{flushleft}
\def\arraystretch{1}
\end{table*}

\begin{table*}
	\renewcommand{\arraystretch}{1.2}
	\centering
	\caption{Average boundary deviation (ABD) for 10 frame segmentation propagation}\label{TB:ABD_RGBD}\vspace{-2mm}
	\small{ \begin{tabular*}{\textwidth}{@{\extracolsep{\fill}}c|cccccc }
			\toprule
			  \multirow{2}{*}{\bf Method} 	&Zhong et al's &+Matting & +Uniform GC& +OCSSVM+prior  &+2CSSVM  &OCSSVM+prior\\ 
			&Classifier & in RGB & in RGBD & in RGB & in RGBD & in RGBD\\\hline 
			{\bf ABD(pixel)} & 11.01& 12.73&3.61&3.35&3.07& \textbf{2.92}
\\			\bottomrule  
		\end{tabular*}}\vspace{-0.2cm}
	\def\arraystretch{1}
\end{table*}

\begin{figure*}[!t]
	\centering
	\includegraphics[width=0.95\linewidth]{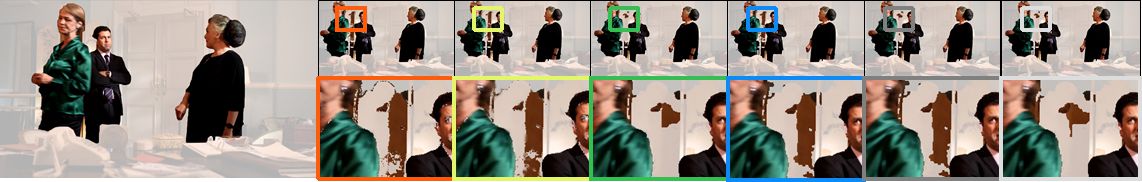}\\
	\includegraphics[width=0.95\linewidth]{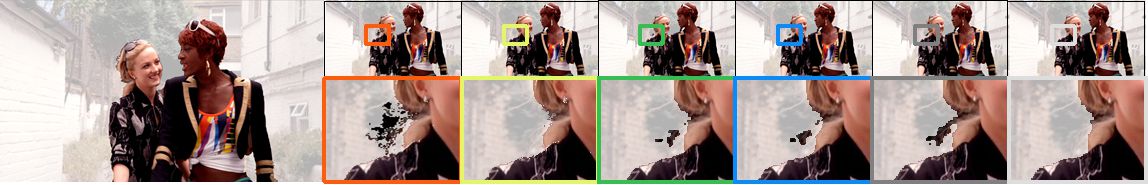}
	\caption{Segmentation results on two RGB-D sequences (zoom in to see details). In each example, the left most is the initial keyframe segmentation. From the second left to the end are the results by Zhong et al.'s classifier \cite{Zhong2012UDC_SIGGRAPHAsia}, Classifier with Matting, OSSVM with RGB only, uniform weights for GC on RGB-D, 2CSSVM for RGB-D and OSSVM for RGB-D. The top row shows the result on the 5th frame from the keyframe, and the bottom row shows the zoom-in for the boxed regions in the top row.}\label{FIG:SeqPropagRGBD}\vspace{-0.5cm}
\end{figure*}

The dataset we used is from the INRIA 3D movie dataset \cite{Seguin153DMovie}. Since a number of sequences in the original dataset contain very dark or motion-blurred objects, we select a subset containing 22 sequences with identifiable object boundaries in RGB domain for this experiment. Note that visually identifiable boundaries are required in the context of video cutout and for ground truth delineation. There are a total of 835 frames in the selected subset. Besides, the original dataset only provides the keyframe segmentations. Therefore, we manually cutout each frame. 

The visual comparison of results from the related methods are shown in Figure \ref{FIG:SeqPropagRGBD}. From the visual results, we can observe that the results are comparable and our OSSVM on RGB-D outperforms others in general. The quantitative results are shown in Figure \ref{FIG:Rst_SeqPropagRGBD} and Table \ref{TB:ABD_RGBD}, which further validated our observation. The first impression is that the overall errors are only around 3 pixels small for many of the methods up to 10 subsequent frames. Besides, we see that our OSSVM on RGB-D and 2CSSVM on RGB-D are very comparable and OSSVM is still better than the other methods. 

The comparable performance is due to that the object/background motion in the videos in this dataset are often either very small or abrupt, and the method may perform equally good or bad on most of them. The RGB-D video cutout system would further benefit from a dedicated RGB-D object classifier which is beyound the scope of this paper.

\begin{figure}[!t]
\centering
\includegraphics[width=0.95\columnwidth]{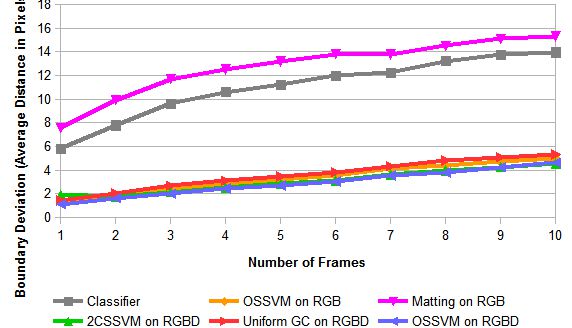}
\caption{Error accumulation in segmentation propagation over RGB-D sequences without additional user interaction. }\label{FIG:Rst_SeqPropagRGBD}\vspace{-0.5cm}
\end{figure}

\subsection{Cross validation for prior weight selection}
There is an optional edge prior in our main formulation of OSSVM. The default prior value, which is the coefficient of $-w^{edge}$, is 1. From our experiment, we observe that this prior is crucial to the performance. To select the optimal weights, we perform 10-fold cross validation with 7 possible prior weights on a uniform grid $\{0,0.5,1,1.5,2,2.5,3\}$, for both of the RGB and RGB-D datasets used in our experiments. The best prior is the one gives overall smallest error in all the 10-fold cross validation. 

The quantitative results are shown in Figure \ref{FIG:CrossValid}. We can observe that the error in the 10-fold experiments becomes smallest for Zhong et al's dataset if the prior is 1 and the error becomes smallest for the RGB-D dataset if the prior is 0.5. When the prior on the edge weight is too high, the weights on the data term may be too small to produce semantically meaningful results. 
\begin{figure}[!t]
 	\centering
 	\includegraphics[width=0.95\columnwidth]{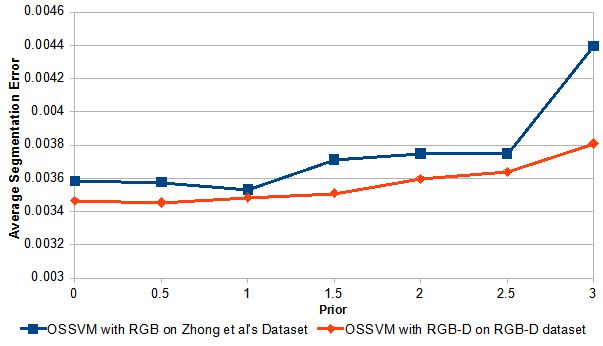}
 	\caption{Cross validation for edge prior weight selection. The average segmenation error is defined as the ratio of wrongly segmented pixels against the total number of pixels.}\label{FIG:CrossValid}
 \end{figure}

\subsection{Limitations}\label{sec:discussion}
Although experiments show that our rectification approach can effectively improve the performance of existing video cutout systems, it may fail in challenging cases. A typical failure case is imperfect edge extraction. State-of-the art edge detection techniques are often reliable, but their results may still contain errors. Such errors may impair the segmentation rectification. See Figure~\ref{FIG:SpuriousEdge} for such an example. The edge map shows that there is a clear strong edge caused by the shadow in-between the legs, causing segmentation rectification to be less effective in this region. 

Another common problem for segmentation propagation is having abrupt changes on the object itself, especially the abrupt emergence of object parts. Figure~\ref{FIG:SuddenParts} shows a typical example. Apart from more user interactions, we believe this problem can be solved by using a more sophisticated propagation model, such as a long-term shape prior.

%% file: discussion_CVPR_revised.tex
\begin{figure}[!h]
	\centering
	\subfloat{\includegraphics[width=0.33\columnwidth]{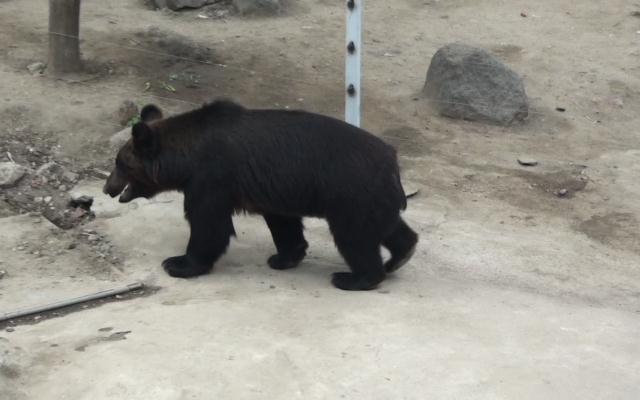}
		\includegraphics[width=0.33\columnwidth]{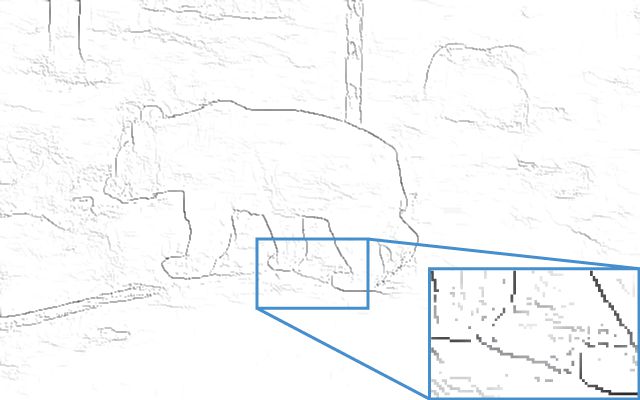}
		\includegraphics[width=0.33\columnwidth]{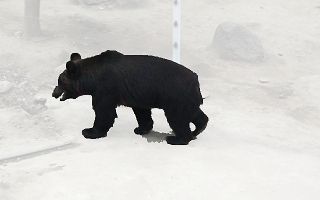}}\\\vspace{-2mm}
	\caption{Examples of failure cases I: spurious edges.}\label{FIG:SpuriousEdge}\vspace{-3mm}
	\centering
	\subfloat{\includegraphics[width=0.33\columnwidth]{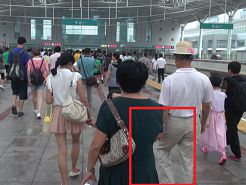}
		\includegraphics[width=0.33\columnwidth]{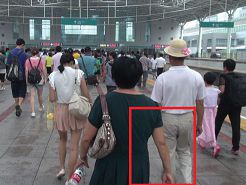}
		\includegraphics[width=0.33\columnwidth]{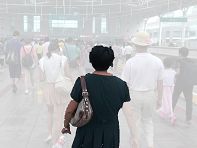}}\\\vspace{-2mm}
	\caption{Examples of failure cases II: abrupt emergence of object parts.}\label{FIG:SuddenParts}
\end{figure}


\section{Conclusion and future work}\label{sec:conclusion}
We propose a novel generic approach to automatically rectify the propagated segmentation in video cutout systems. The core idea of our work is to incorporate a generic shape distance measure in a bilayer MRF framework learned from data to remove the intrinsic bias of the classifier in the segmentation propagation step. This work is motivated by our observation that different classifiers bias toward FP and FN differently, but they were treated equally in the previous video cutout systems. We found that FP and FN can be treated differently in our bilayer MRF, and the optimal form of the MRF can be learned from the data. Extensive evaluation demonstrates that our approach can significantly improve the state-of-the-art video cutout systems in segmentation accuracy.  

There are several vision tasks related to the interactive video cutout problem, such as~\cite{Xiao2005motionseg,lee2011key,MinglunGong2011VideoSeg,Papazoglou2013fastVseg,chen2013deepshapeprior,Jain2014supervoxelVseg,khoreva2015classifierVseg}. In computer vision, those problems can be thought of as shape tracking problem, and errors are generally tolerable. Convenient user interaction is also not a concern to them. In contrast, the interactive video cutout problem does not tolerate visible errors in the segmentation and user-friendly interaction is a crucial concern. It has been proven that the state-of-the-art video cutout frameworks are particularly suitable for the video cutout problem. In most of the works in vision, global optimization frameworks for whole sequences are often adopted. It has been noted in~\cite{Bai09VideoSnapCut_SIGGRAPH} that localized optimization allows user to have better control of the video cutout process. Our method is proposed dedicatedly for video cutout. Its extension for generic video segmentation tasks has yet to be explored.

\section*{Appendix}

\begin{proof}[Proof of Theorem \ref{THM:2COEQ}]
	We first substitute the two conditions stated in the theorem into the constraint in Eq. (\ref{EQ:2C_SSVM1}) and we obtain	
	\begin{equation}
	\begin{split}
	&\mathbf{w}\cdot(\mathbf{\Psi}(\mathbf{x}_k,\mathbf{f}_k)-\mathbf{\Psi}(\mathbf{x}_k,\mathbf{f}^*_k))\geq  \Delta_k - \xi_k\\
	\Leftrightarrow&\mathbf{w}\cdot(\mathbf{b}_k-\mathbf{\Psi}(\mathbf{x}_k,\mathbf{f}^*_k))\geq  1 - \xi_k\\
	\Leftrightarrow&\mathbf{w}\cdot\mathbf{\Psi}(\mathbf{x}_k,\mathbf{f}_k^*)\leq  -1 +\hat{\xi}_k\\
	\end{split}
	\end{equation}
	where have applied the normalization constraint $\sum_i w_i=1$ and we set $\hat{\xi}_k = {B}_k+ \xi_k$, ${B}_k=\sum_j[\mathbf{b}_k]_j$. The above gives us the constraint in the OSSVM formulation in Eq. (\ref{EQ:OSSVM}).
	
	The objective function in Eq. (\ref{EQ:2C_SSVM1}) can accordingly be rewritten as
	\begin{equation}
	\begin{split}
	&{1\over2}\|\mathbf{w}\|^2 + {C\over N}\sum_{k=1}^N (\hat{\xi}_k-{B}_k)\\
	=& {1\over2}\|\mathbf{w}\|^2 + {C\over N}\sum_{k=1}^N \hat{\xi}_k - {C\over N}\sum_{k=1}^N{B}_k\\
	\end{split}
	\end{equation}
	where the last term is a constant. By omitting the constant we obtain the objective function in OSSVM  Eq. (\ref{EQ:OSSVM}). 
	
	Lastly, due to the one-to-one correspondence between $\hat{\xi}_k$ and $\xi_k$ for any $k$, we know that optimizing over $\mathbf{w}$ and $\{\xi_k\}$ is equivalent to optimizing over $\mathbf{w}$ and $\{\hat{\xi}_k\}$, which completes the proof. \qed
\end{proof}